%% file: arxiv.tex
\documentclass{article}

\usepackage[letterpaper, left=0.9in, right=0.9in, top=0.8in, bottom=1in]{geometry}

\usepackage[sort]{natbib}
\input{NeurIPS_26/preamble}

\title{\textbf{Taking the Road Less Scheduled \\ with Adaptive Polyak Steps}}

\author{
    Dimitris Oikonomou\\
    Johns Hopkins University\\
    \texttt{doikono1@jh.edu} 
    \and 
    Matthew Buchholz\\
    University of British Columbia\\ 
    \texttt{mdbuch@cs.ubc.ca}
    \and 
    Yuen-Man Pun\\
    Australian National University\\
    \texttt{yuenman.pun@anu.edu.au}
    \bigskip
    \and 
    Robert M. Gower\\
    Center for Computational Mathematics \\ Flatiron Institute, Simons Foundation\\
    \texttt{gowerrobert@gmail.com} 
    \and 
    Nicolas Loizou\\
    Johns Hopkins University\\
    \texttt{nloizou@jhu.edu}
}

\begin{document}

\maketitle

\begin{abstract}
    Schedule-Free SGD, proposed in \emph{The Road Less Scheduled} \citep{DYK+24}, achieves optimal convergence rates without requiring the training horizon in advance, by replacing learning rate schedules with a principled form of iterate averaging. However, the method still requires tuning a base learning rate whose optimal value depends on unknown problem constants. In this work, we continue down this road by deriving Polyak-type step sizes for Schedule-Free SGD and Adam that compute the learning rate at each iteration from the sampled loss, gradient, and current iterates alone. We first propose an oracle variant that uses per-sample optimal function values and prove an $\mathcal{O}(1/\sqrt{t})$ anytime last-iterate rate for convex Lipschitz objectives. We then remove the oracle requirement with a safeguarded variant that replaces the unknown optimal values with any available lower bound, achieving the same rate up to a neighborhood that vanishes under interpolation. Both step sizes reduce to existing Polyak rules for standard SGD when momentum is set to zero, unifying standard and schedule-free Polyak methods. Numerical experiments on language modeling, including pretraining and distillation, show that the proposed methods match or surpass tuned Schedule-Free baselines while offering greater robustness to hyperparameter choices.
\end{abstract}

\section{Introduction}
\label{sec:intro}

We consider the stochastic minimization problem
\begin{align}
    \label{eq:main-prob}
    \min_{\bm{x}\in\R^d}\;f(\bm{x}):=\mathbb{E}_{\bm{\zeta}}\!\left[f_{\bm{\zeta}}(\bm{x})\right],
\end{align}
where each $f_{\bm{\zeta}}:\R^d\to\R$ is convex and differentiable. We assume that the set of minimizers $X^*:=\arg\min f$ is nonempty and denote $\bm{x}_\star\in X^*$, $f^\star=f(\bm{x}_\star)$. 

A central practical challenge in training machine learning models is the choice of the learning rate and its schedule. Standard practice relies on learning rate schedules, such as cosine annealing \citep{loshchilov2017sgdr}, linear decay~\citep{HBK+24}, or step-wise reduction~\citep{lecun-gradientbased-learning-applied-1998}, that describe how the rate evolves over a pre-specified training horizon $T$. While highly effective, these schedules share a fundamental limitation: they require the total number of training steps to be fixed in advance, and if training is stopped early or extended beyond $T$ the scheduler is suboptimal. Classical stochastic approximation theory offers a principled alternative through Polyak-Ruppert iterate averaging \citep{polyak1990new,ruppert1988efficient}, which achieves minimax-optimal convergence rates without knowledge of $T$. Despite this theoretical appeal, iterate averaging dramatically underperforms in practice \citep{DYK+24}, creating a long-standing gap between theory and application. Recent work has made significant progress toward closing this gap by developing methods that implicitly apply a schedule through a careful form of iterate averaging, thereby matching or surpassing the practical performance of schedules without requiring the stopping time to be known in advance. Schedule-Free SGD, proposed by \citet{DYK+24}, is a prominent example of this approach. 

Schedule-Free SGD maintains three sequences: a \emph{gradient evaluation} sequence $\bm{y}_t$, a \emph{base} sequence $\bm{z}_t$ where the gradient step is applied, and an \emph{iterate average} sequence $\bm{x}_t$ that is returned as output. At each iteration $t$, the update is
\begin{align}
    \label{eq:sf}
    \bm{y}_t&=(1-\beta)\bm{z}_{t-1}+\beta\bm{x}_t, \nonumber\\
    \bm{z}_t&=\bm{z}_{t-1}-\gamma_t\nabla f_{\bm{\zeta}_t}(\bm{y}_t), \tag{Schedule-Free SGD}\\
    \bm{x}_{t+1}&=(1-c_{t+1})\bm{x}_t+c_{t+1}\bm{z}_t, \nonumber
\end{align}
with $c_{t+1}=1/(t+1)$ and $\bm{z}_{-1}=\bm{x}_0$. The momentum parameter $\beta\in[0,1]$ interpolates between Polyak-Ruppert averaging ($\beta=0$) and Primal averaging ($\beta=1$); values near $\beta\approx0.9$ work well in practice \citep{DYK+24}. When each $f_{\bm{\zeta}}$ is $G$-Lipschitz, \ref{eq:sf} with a constant step size $\gamma_t\equiv\gamma$ achieves the optimal $\O(DG/\sqrt{T})$ rate for any $\beta\in[0,1)$ as shown in \citet{DYK+24}, where $D=\norm{\bm{x}_0-\bm{x}_\star}$. The method has demonstrated strong empirical performance across diverse benchmarks, including winning the MLCommons 2024 AlgoPerf Algorithmic Efficiency Challenge Self-Tuning track\footnote{Schedule-Free AdamW was the winning entry, see \url{https://mlcommons.org/2024/08/mlc-algoperf-benchmark-competition/}}. Recently, \citet{brown25} established convergence of \ref{eq:sf} in the smooth nonconvex setting for the primal averaging case ($\beta=1$) with a constant learning rate. The method is also closely related to the AC-SA algorithm of \citet{lan12}, which achieves optimal rates via a similar averaging scheme. For more additional related work, see \Cref{app:related-polyak}.

\begin{table}[h]
    \centering
    \caption{Summary of convergence guarantees for \ref{eq:sf} under different step size rules. All rates are for convex and Lipschitz objectives and refer to the last iterate $\bm{x}_t$. ``Oracle-free'' indicates that the step size does not require knowledge of $f_{\bm{\zeta}}(\bm{x}_\star)$.}
    \resizebox{\textwidth}{!}{%
    \begin{tabular}{lcccc}
        \toprule
        \textbf{Method} & \textbf{Step size $\gamma_t$} & \textbf{Oracle-free?} & \textbf{Anytime?} & \textbf{Rate} \\
        \midrule
        \ref{eq:sf} \citep{DYK+24} & Constant & \cmark & \xmark & $\O(1/\sqrt{T})$ \\
        \midrule
        \rowcolor{green!15}
        \ref{eq:sf} + \ref{eq:sf-ps} (\S\ref{sec:polyak-sf-oracle}) & \emph{Adaptive} & \xmark & \cmark & $\O(1/\sqrt{t})$ \\
        \rowcolor{green!15}
        \ref{eq:sf} + \ref{eq:sf-sps-safe} (\S\ref{sec:polyak-sf-safe}) & \emph{Adaptive} & \cmark & \cmark & $\O(1/\sqrt{t}+\sigma^2)$ \\
        \bottomrule
    \end{tabular}%
    }
    \label{tab:comparison}
\end{table}

Despite eliminating the dependence on the stopping time $T$, \ref{eq:sf} still requires tuning the base learning rate $\gamma$. The theoretically optimal choice $\gamma=D/(G\sqrt{T})$ depends on the unknown initial distance $D$, the Lipschitz constant $G$, and the very quantity $T$ that the method was designed to avoid. In practice, $\gamma$ is selected by grid search, which can be costly when combined with sweeps over weight decay and momentum. 

Instead of tuning $\gamma$, here we propose to replace the constant $\gamma$ with a \emph{Polyak-type} step size \citep{polyak1969minimization}, which calculates the learning rate at each iteration from function values rather than from problem constants. The classical Polyak step size and its stochastic extensions \citep{BZK20,LVLL21,garrigos2023function,oikonomou2025safeguarded} have proven effective for SGD, achieving optimal convergence rates with minimal tuning. In this work, we derive Polyak-type step sizes tailored to the \ref{eq:sf} update, making the method fully adaptive. See \Cref{tab:comparison} for a summary.
 
\subsection{Main Contributions}
\label{sec:main-contrib}

Our contributions are summarized as follows.

\paragraph{Polyak step sizes for \ref{eq:sf}.}
We derive Polyak-type step sizes for \ref{eq:sf} by minimizing an appropriate upper bound on the distance to the solution. The resulting step size is closed-form and extends the classical SPS$_+$ rule of \citet{garrigos2023function}. To remove its dependence on the per-sample optimal loss, we introduce a safeguarded variant following the strategy of \citet{oikonomou2025safeguarded}. In summary, we propose:
\begin{itemize}[leftmargin=*,itemsep=2pt]
    \item \textbf{Oracle (\ref{eq:sf-ps}).} Depends only on the sampled loss $f_{\bm{\zeta}_t}(\bm{y}_t)$, the gradient $\nabla f_{\bm{\zeta}_t}(\bm{y}_t)$, and the current iterates, but requires knowledge of the per-sample optimal loss $f_{\bm{\zeta}_t}(\bm{x}_\star)$.
    \item \textbf{Safeguarded (\ref{eq:sf-sps-safe}).} Replaces $f_{\bm{\zeta}_t}(\bm{x}_\star)$ with any available lower bound $\ell_{\bm{\zeta}_t}^*$ and prevents exploding steps via a $\max\{\norm{\nabla f_{\bm{\zeta}_t}(\bm{y}_t)}^2,M\}$ safeguard in the denominator. The resulting step size is fully adaptive: it requires no knowledge of the Lipschitz constant $G$, the initial distance $\norm{\bm{x}_0-\bm{x}_\star}$, or the stopping time $T$.
\end{itemize}
Setting $\beta=0$ in either variant recovers the corresponding SGD rule of \citet{garrigos2023function,oikonomou2025safeguarded}, thereby unifying standard and schedule-free Polyak methods within a single framework.

\paragraph{Convergence guarantees.}
For convex $G$-Lipschitz objectives, we establish \emph{last-iterate} convergence guarantees for both Polyak step size variants (\Cref{sec:convergence}). For the oracle step size (\ref{eq:sf-ps}), we prove an $\O(1/\sqrt{t})$ anytime rate that matches the optimum up to constants. For the safeguarded step size (\ref{eq:sf-sps-safe}), we prove $\O(1/\sqrt{t})$ convergence up to a neighborhood controlled by a variance measure $\sigma^2$; this neighborhood vanishes under interpolation. To the best of our knowledge, these are the first convergence guarantees for adaptive \ref{eq:sf} based on Polyak-type step sizes.

\paragraph{Practical variants.}
We extend the proposed Schedule-Free Polyak step sizes in two practical directions. First, we derive preconditioned Adam-based variants (\ref{eq:sf-adam-sps-plus} and \ref{eq:sf-adam-sps-safe}) by replacing the Euclidean gradient norm in the denominator with the norm induced by the inverse preconditioner $\mD_t^{-1}$. These variants recover the SGD rules when $\mD_t=\mI$. Second, we introduce an adaptive safeguard that replaces the fixed constant $M$ with an exponential moving average of squared gradient norms, reducing the need to tune this parameter. We evaluate these practical variants alongside the SGD-based methods in our experiments.

\paragraph{Numerical evaluation.} 
We complement our theoretical results with an empirical study across black-box model distillation, and language model pretraining. On model distillation, where the teacher's loss provides a natural approximation of the oracle value, we distill large language models into smaller students using the oracle step size (\ref{eq:sf-ps}). On language modeling, we pretrain a GPT-2 decoder on FineWeb 10B using the Adam-based safeguarded variant. Our results indicate that the proposed Polyak step sizes match or surpass tuned baselines in validation loss while substantially reducing sensitivity to the choice of learning rate, supporting the practical relevance of fully adaptive schedule-free optimization.

\section{Polyak Step Sizes for Schedule-Free SGD}
\label{sec:polyak-sf}

This section derives Polyak-type step sizes for \ref{eq:sf}. We begin by recalling two existing Polyak step size rules for SGD: the oracle-based SPS$_+$ of \citet{garrigos2023function} and the safeguarded variant \ref{eq:sps-safe} of \citet{oikonomou2025safeguarded}. We then show how both rules extend naturally to the Schedule-Free setting: first the oracle version (\Cref{sec:polyak-sf-oracle}), then the fully adaptive safeguarded version (\Cref{sec:polyak-sf-safe}). The correspondence between the SGD and Schedule-Free rules is summarized in \Cref{tab:correspondence}.  

\begin{table*}[h]
    \centering
    \caption{Correspondence between Polyak step sizes for classical SGD and \ref{eq:sf}. Here $\ell_{\bm{\zeta}_t}^*$ is any lower bound on $f_{\bm{\zeta}_t}$ and $M>0$ is the safeguard constant. Setting $\beta=0$ in the \ref{eq:sf} column recovers the corresponding SGD rule, since $\bm{y}_t=\bm{z}_{t-1}$.}
    \resizebox{\textwidth}{!}{%
    \begin{tabular}{l|cc}
        \toprule
        \textbf{Setting} & \textbf{SGD} & \textbf{\ref{eq:sf}} \\
        \midrule
        Oracle 
        & $\displaystyle \gamma_t=\frac{\left[f_{\bm{\zeta}_t}(\bm{x}^t)-f_{\bm{\zeta}_t}(\bm{x}_\star)\right]_+}{\norm{\nabla f_{\bm{\zeta}_t}(\bm{x}^t)}^2}$ 
        & $\displaystyle \gamma_t=\frac{\left[f_{\bm{\zeta}_t}(\bm{y}_t)-f_{\bm{\zeta}_t}(\bm{x}_\star)+\beta\dotprod{\nabla f_{\bm{\zeta}_t}(\bm{y}_t),\bm{z}_{t-1}-\bm{x}_t}\right]_+}{\norm{\nabla f_{\bm{\zeta}_t}(\bm{y}_t)}^2}$ \\[14pt]
        Safeguarded 
        & $\displaystyle \gamma_t=\frac{f_{\bm{\zeta}_t}(\bm{x}^t)-\ell_{\bm{\zeta}_t}^*}{\max\{\norm{\nabla f_{\bm{\zeta}_t}(\bm{x}^t)}^2,M\}}$ 
        & $\displaystyle \gamma_t=\frac{\left[f_{\bm{\zeta}_t}(\bm{y}_t)-\ell_{\bm{\zeta}_t}^*+\beta\dotprod{\nabla f_{\bm{\zeta}_t}(\bm{y}_t),\bm{z}_{t-1}-\bm{x}_t}\right]_+}{\max\{\norm{\nabla f_{\bm{\zeta}_t}(\bm{y}_t)}^2,M\}}$ \\
        \bottomrule
    \end{tabular}
    }
    \label{tab:correspondence}
\end{table*}

\subsection{Background: From SPS$_+$ to SPS$_{\text{safe}}$}
\label{sec:polyak-background}

A classical strategy for choosing the step size in SGD, given by $\bm{x}^{t+1}=\bm{x}^t-\gamma_t\nabla f_{\bm{\zeta}_t}(\bm{x}^t)$, is to select $\gamma_t$ that minimizes an upper bound on the quantity $\|\bm{x}^{t+1}-\bm{x}_\star\|^2$. When $f_{\bm{\zeta}_t}$ is convex, expanding the squared distance and bounding the inner product via convexity yields a quadratic in $\gamma_t$ whose minimizer over $\gamma_t\geq0$ is given by the SPS$_+$ rule, as introduced in \citet{garrigos2023function}:
\begin{align}
    \label{eq:sps-plus}
    \gamma_t=\frac{\left[f_{\bm{\zeta}_t}(\bm{x}^t)-f_{\bm{\zeta}_t}(\bm{x}_\star)\right]_+}{\norm{\nabla f_{\bm{\zeta}_t}(\bm{x}^t)}^2},
    \tag{SPS$_+$}
\end{align}
where $[z]_+=\max\{z,0\}$ is the ReLU function, and guarantees that $\gamma_t\geq0$. Under this rule, the iterates are Fej\'er monotone with respect to $\bm{x}_\star$, and for convex $G$-Lipschitz objectives \ref{eq:sps-plus} achieves the optimal $\mathcal{O}(G\norm{\bm{x}_0-\bm{x}_\star}/\sqrt{T})$ rate without requiring interpolation \citep{garrigos2023function}. The price is that \ref{eq:sps-plus} requires oracle access to $f_{\bm{\zeta}_t}(\bm{x}_\star)$, which is rarely available in practice. Earlier variants such as SPS$_{\max}$ of \citet{LVLL21} sidestep this by using the infima $f_{\bm{\zeta}_t}^*:=\inf f_{\bm{\zeta}_t}$, but then require the interpolation condition $f_{\bm{\zeta}}(\bm{x}_\star)=\inf_{\bm{x}}f_{\bm{\zeta}}(\bm{x})$ for all $\bm{\zeta}$, in order to establish convergence guarantees in the $G$-Lipschitz setting.

To remove both the oracle dependence and the interpolation requirement, \citet{oikonomou2025safeguarded} introduced the safeguarded SPS:
\begin{align}
    \label{eq:sps-safe}
    \gamma_t=\frac{f_{\bm{\zeta}_t}(\bm{x}^t)-\ell_{\bm{\zeta}_t}^*}{\max\{\norm{\nabla f_{\bm{\zeta}_t}(\bm{x}^t)}^2,M\}},
    \tag{SPS$_{\text{safe}}$}
\end{align}
where $\ell_{\bm{\zeta}_t}^*\leq f_{\bm{\zeta}_t}(\bm{x}_\star)$ is any available lower bound (typically $\ell_{\bm{\zeta}_t}^*=0$ for non-negative losses) and $M>0$ is a safeguard constant. Two things change relative to \ref{eq:sps-plus}: the unknown $f_{\bm{\zeta}_t}(\bm{x}_\star)$ is replaced by the accessible $\ell_{\bm{\zeta}_t}^*$, and $\max\{\cdot,M\}$ in the denominator prevents the step size from exploding when $\norm{\nabla f_{\bm{\zeta}_t}(\bm{x}^t)}$ is small. The result is $\mathcal{O}(1/\sqrt{t})$ convergence to a neighborhood of the solution, with the neighborhood vanishing under interpolation \citep{oikonomou2025safeguarded}. We now extend both \ref{eq:sps-plus} and \ref{eq:sps-safe} to \ref{eq:sf}.

\subsection{Oracle Polyak Step Size for \ref{eq:sf}}
\label{sec:polyak-sf-oracle}

We now carry out the same upper-bound argument for the Schedule-Free $\bm{z}$ sequence. Recall that the Schedule-Free update sets $\bm{z}_t=\bm{z}_{t-1}-\gamma_t\nabla f_{\bm{\zeta}_t}(\bm{y}_t)$. Expanding the squared distance to $\bm{x}_\star$ gives
\begin{align}
    \label{eq:expand}
    \norm{\bm{z}_t-\bm{x}_\star}^2
    =\norm{\bm{z}_{t-1}-\bm{x}_\star}^2-2\gamma_t\dotprod{\nabla f_{\bm{\zeta}_t}(\bm{y}_t),\bm{z}_{t-1}-\bm{x}_\star}+\gamma_t^2\norm{\nabla f_{\bm{\zeta}_t}(\bm{y}_t)}^2.
\end{align}
The key difference from the standard subgradient setting is that the gradient is evaluated at $\bm{y}_t$, not at $\bm{z}_{t-1}$. At this point, we cannot minimize the right-hand side in $\gamma_t$, because the minima would then depend explicitly on $\bm{x}_\star$ which appears in the inner product. To remove this explicit dependency on $\bm{x}_\star$, we will use convexity to bound this inner product. 

To apply convexity, we need to decompose $\bm{z}_{t-1}-\bm{x}_\star$ in terms of $\bm{y}_t$. Rearranging the interpolation update $\bm{y}_t=(1-\beta)\bm{z}_{t-1}+\beta\bm{x}_t$ gives
\begin{align}
    \label{eq:zt-yt}
    \bm{z}_{t-1}= \bm{y}_t+\beta (\bm{z}_{t-1}-\bm{x}_t)
\end{align}
which substituting into the inner product yields 
\begin{align*}
    \dotprod{\nabla f_{\bm{\zeta}_t}(\bm{y}_t),\bm{z}_{t-1}-\bm{x}_\star}
    =\dotprod{\nabla f_{\bm{\zeta}_t}(\bm{y}_t),\bm{y}_t-\bm{x}_\star}+\beta\dotprod{\nabla f_{\bm{\zeta}_t}(\bm{y}_t),\bm{z}_{t-1}-\bm{x}_t}.
\end{align*}
Now convexity of $f_{\bm{\zeta}_t}$ can be applied to the first of the above two terms on the right giving
$$\dotprod{\nabla f_{\bm{\zeta}_t}(\bm{y}_t),\bm{y}_t-\bm{x}_\star}\geq f_{\bm{\zeta}_t}(\bm{y}_t)-f_{\bm{\zeta}_t}(\bm{x}_\star).$$
Substituting back into \eqref{eq:expand} produces the upper bound
\begin{align}
    \label{eq:upper-bound}
    \norm{\bm{z}_t-\bm{x}_\star}^2
    \leq\norm{\bm{z}_{t-1}-\bm{x}_\star}^2+\gamma_t^2\norm{\nabla f_{\bm{\zeta}_t}(\bm{y}_t)}^2-2\gamma_t\left(f_{\bm{\zeta}_t}(\bm{y}_t)-f_{\bm{\zeta}_t}(\bm{x}_\star) +\beta\dotprod{\nabla f_{\bm{\zeta}_t}(\bm{y}_t),\bm{z}_{t-1}-\bm{x}_t}\right).
\end{align}
Minimizing the right-hand side over $\gamma_t\geq0$ gives the \emph{Schedule-Free Polyak Step}:
\begin{align}
    \label{eq:sf-ps}
    \boxed{\gamma_t=\frac{\left[f_{\bm{\zeta}_t}(\bm{y}_t)-f_{\bm{\zeta}_t}(\bm{x}_\star)+\beta\dotprod{\nabla f_{\bm{\zeta}_t}(\bm{y}_t),\bm{z}_{t-1}-\bm{x}_t}\right]_+}{\norm{\nabla f_{\bm{\zeta}_t}(\bm{y}_t)}^2}}
    \tag{SF-SPS$_+$}
\end{align}
The ReLU function $[\cdot]_+$ guarantees $\gamma_t\geq0$; it is active only when the momentum correction $\beta\dotprod{\nabla f_{\bm{\zeta}_t}(\bm{y}_t),\bm{z}_{t-1}-\bm{x}_t}$ is sufficiently negative to overcome $f_{\bm{\zeta}_t}(\bm{y}_t)-f_{\bm{\zeta}_t}(\bm{x}_\star)$. In the setting where $\beta=0$, \ref{eq:sf-ps} becomes \ref{eq:sps-plus}, since in that case $\bm{y}_t=\bm{z}_{t-1}$ and the inner product term vanishes. 

\ref{eq:sf-ps} still requires the oracle value $f_{\bm{\zeta}_t}(\bm{x}_\star)$; we next introduce a safeguarded variant that replaces it with a lower bound. The complete method for both variants is presented in \Cref{alg:sf-sps}.

\subsection{Safeguarded Polyak Step Size for \ref{eq:sf}}
\label{sec:polyak-sf-safe}

Following the same strategy that led from \ref{eq:sps-plus} to \ref{eq:sps-safe} in the standard subgradient setting, we apply the same two modifications to remove the oracle dependence of \ref{eq:sf-ps}: (i) the unknown $f_{\bm{\zeta}_t}(\bm{x}_\star)$ is replaced by a lower bound $\ell_{\bm{\zeta}_t}^*\leq f_{\bm{\zeta}_t}(\bm{x}_\star)$, and (ii) the denominator $\norm{\nabla f_{\bm{\zeta}_t}(\bm{y}_t)}^2$ is replaced by $\max\{\norm{\nabla f_{\bm{\zeta}_t}(\bm{y}_t)}^2,M\}$ for a user-chosen safeguard $M>0$. This yields the \emph{Safeguarded Schedule-Free Polyak Step}:
\begin{align}
    \label{eq:sf-sps-safe}
    \boxed{\gamma_t=\frac{\left[f_{\bm{\zeta}_t}(\bm{y}_t)-\ell_{\bm{\zeta}_t}^*+\beta\dotprod{\nabla f_{\bm{\zeta}_t}(\bm{y}_t),\bm{z}_{t-1}-\bm{x}_t}\right]_+}{\max\{\norm{\nabla f_{\bm{\zeta}_t}(\bm{y}_t)}^2,M\}}}
    \tag{SF-SPS$_{\text{safe}}$}
\end{align}
The complete method is presented in \Cref{alg:sf-sps}.

\begin{algorithm}[h]
\begin{algorithmic}[1]
    \caption{\ref{eq:sf} with \ref{eq:sf-ps} or \ref{eq:sf-sps-safe}}
    \label{alg:sf-sps}
    \State \textbf{Input:} $\bm{z}_{-1}=\bm{x}_0\in\R^d$, $\beta\in[0,1)$, lower bounds $\ell_{\bm{\zeta}_t}^*$, safeguard $M>0$
    \For{$t=0$ \textbf{to} $T-1$}
        \State $\bm{y}_t=(1-\beta)\bm{z}_{t-1}+\beta\bm{x}_t$
        \State Sample $\bm{\zeta}_t$; compute $\nabla f_{\bm{\zeta}_t}(\bm{y}_t)$
        \State \textbf{Option I:} $\displaystyle\gamma_t=   \frac{\left[f_{\bm{\zeta}_t}(\bm{y}_t)-f_{\bm{\zeta}_t}(\bm{x}_\star)+\beta\dotprod{\nabla f_{\bm{\zeta}_t}(\bm{y}_t),\bm{z}_{t-1}-\bm{x}_t}\right]_+}{\norm{\nabla f_{\bm{\zeta}_t}(\bm{y}_t)}^2}  $ \hfill $\triangleright$ \ref{eq:sf-ps}
        \Statex \hspace{\algorithmicindent}\textbf{or}
        \State \textbf{Option II:} $\displaystyle\gamma_t=\frac{\left[f_{\bm{\zeta}_t}(\bm{y}_t)-\ell_{\bm{\zeta}_t}^*+\beta\dotprod{\nabla f_{\bm{\zeta}_t}(\bm{y}_t),\bm{z}_{t-1}-\bm{x}_t}\right]_+}{\max\{\norm{\nabla f_{\bm{\zeta}_t}(\bm{y}_t)}^2,M\}}$ \hfill $\triangleright$ \ref{eq:sf-sps-safe}
        \State $\bm{z}_t=\bm{z}_{t-1}-\gamma_t\nabla f_{\bm{\zeta}_t}(\bm{y}_t)$
        \State $\bm{x}_{t+1}=(1-c_{t+1})\bm{x}_t+c_{t+1}\bm{z}_t$, \quad $c_{t+1}=\frac{1}{t+1}$
    \EndFor
    \State \textbf{Return:} $\bm{x}_T$
\end{algorithmic}
\end{algorithm}

As in the oracle version, when $\beta>0$, the inner product term $\beta\dotprod{\nabla f_{\bm{\zeta}_t}(\bm{y}_t),\bm{z}_{t-1}-\bm{x}_t}$ can be negative, potentially making the numerator negative even though $f_{\bm{\zeta}_t}(\bm{y}_t)-\ell_{\bm{\zeta}_t}^*\geq0$; the $[\cdot]_+$ operator handles this by clamping the step size to zero. In practice, however, this rarely occurs.

When $\beta=0$, we have $\bm{y}_t=\bm{z}_{t-1}$ and the momentum correction vanishes, so \ref{eq:sf-sps-safe} reduces to $[f_{\bm{\zeta}_t}(\bm{z}_{t-1})-\ell_{\bm{\zeta}_t}^*]_+/\allowbreak\max\{\norm{\nabla f_{\bm{\zeta}_t}(\bm{z}_{t-1})}^2,M\}$. Since $f_{\bm{\zeta}_t}(\bm{z}_{t-1})-\ell_{\bm{\zeta}_t}^*\geq0$, the ReLU function is redundant and we recover exactly \ref{eq:sps-safe} applied to the $\bm{z}$ sequence, with $\bm{x}_t$ serving as its Polyak-Ruppert average. 

Meanwhile, $\max\{\norm{\nabla f_{\bm{\zeta}_t}(\bm{y}_t)}^2,M\}$ in the denominator prevents the step size from growing unboundedly when the gradient is small. Unlike the SPS$_{\max}$ approach of \citet{LVLL21}, which clips the \emph{entire} Polyak step size via $\min\{\gamma_t,\gamma_b\}$ and can reduce the method to a near-constant step size in practice \citep{oikonomou2025safeguarded}, our safeguard controls only the denominator and preserves the adaptive numerator throughout training.

Altogether, \ref{eq:sf-sps-safe} depends only on quantities available at iteration $t$, the sampled loss $f_{\bm{\zeta}_t}(\bm{y}_t)$, a lower bound $\ell_{\bm{\zeta}_t}^*$, the gradient $\nabla f_{\bm{\zeta}_t}(\bm{y}_t)$, and the current iterates, and requires no knowledge of the Lipschitz constant $G$, the initial distance $\norm{\bm{x}_0-\bm{x}_\star}$, or the stopping time $T$.

\section{Convergence Analysis}
\label{sec:convergence}
 
This section presents convergence guarantees for \ref{eq:sf} equipped with the Polyak step sizes derived in \Cref{sec:polyak-sf}. We first analyze the oracle step size \ref{eq:sf-ps}, which yields a clean $\O(1/\sqrt{t})$ last-iterate rate but requires knowledge of $f_{\bm{\zeta}_t}(\bm{x}_\star)$ (\Cref{sec:conv-oracle}). We then turn to the safeguarded step size \ref{eq:sf-sps-safe}, which removes this oracle dependence at the cost of convergence to a neighborhood (\Cref{sec:conv-safe}). Complete proofs are deferred to the appendix.

\subsection{Convergence with Oracle Step Size}
\label{sec:conv-oracle}

\begin{theorem}
    \label{thm:oracle}
    Consider the iterates of \ref{eq:sf} with the oracle step size \ref{eq:sf-ps} and $c_t=1/(t+1)$. Let $\beta\in[0,1)$ and let $f_{\bm{\zeta}}:\R^d\to\R$ be convex and differentiable for every $\bm{\zeta}$. Define
    \begin{align*}
        B & :=\{\bm{x}\in\R^d:\norm{\bm{x}-\bm{x}_\star}\leq\norm{\bm{x}_0-\bm{x}_\star}\} \\
        G^2& :=\max_{\bm{x}\in B}\;\E{\norm{\nabla f_{\bm{\zeta}}(\bm{x})}^2}.
    \end{align*}
    With the initialization $\bm{z}_{-1}=\bm{x}_0$, the last iterate satisfies, for all $t\geq1$,
    \begin{align}
        \label{eq:oracle-rate}
        \E{f(\bm{x}_t)-f(\bm{x}_\star)}\leq\frac{G\norm{\bm{x}_0-\bm{x}_\star}}{\sqrt{t+1}}.
    \end{align}
\end{theorem}

The bound \eqref{eq:oracle-rate} has two notable features. First, it is an \emph{anytime} result: it achieves the $\O(1/\sqrt{t})$ rate at every iteration $t$, without requiring the stopping time $T$ to be specified in advance. This is in contrast to the constant-step analysis of \citet{DYK+24}, where the optimal choice $\gamma=D/(G\sqrt{T})$ depends on $T$. Second, the constant $G$ is defined locally over the ball $B$ rather than globally. The proof establishes that the iterates $\bm{z}_t$ are Fej\'er monotone with respect to $\bm{x}_\star$, i.e., $\norm{\bm{z}_t-\bm{x}_\star}\leq\norm{\bm{x}_0-\bm{x}_\star}$ for all $t$. Since $\bm{x}_t$ is a convex combination of the $\bm{z}$ iterates and $\bm{y}_t$ is a convex combination of $\bm{z}_{t-1}$ and $\bm{x}_t$, all three sequences remain in $B$, and consequently the gradients evaluated during training satisfy $\E{\norm{\nabla f_{\bm{\zeta}}(\bm{y}_t)}^2}\leq G^2$.

The rate \eqref{eq:oracle-rate} matches, up to a constant factor, the optimal rate $\O(DG/\sqrt{T})$ for convex Lipschitz optimization \citep{orabona2019modern}. The practical limitation is that \ref{eq:sf-ps} requires oracle access to $f_{\bm{\zeta}_t}(\bm{x}_\star)$. While this quantity is available in some settings, for instance, under interpolation where $f_{\bm{\zeta}}(\bm{x}_\star)=\inf_{\bm{x}}f_{\bm{\zeta}}(\bm{x})$, or in black-box model distillation, it is unavailable in general. This motivates the safeguarded analysis below, which replaces the oracle value with a lower bound at the cost of convergence to a neighborhood.

\subsection{Convergence with Safeguarded Step Size}
\label{sec:conv-safe}

The gap between the lower bound $\ell_{\bm{\zeta}_t}^*$ and the true optimal value $f_{\bm{\zeta}_t}(\bm{x}_\star)$ introduces a noise term that prevents exact convergence. Following the stochastic Polyak step size literature \citep{LVLL21,oikonomou2025safeguarded}, we quantify this gap via the variance measure
\begin{align}
    \label{eq:variance}
    \sigma^2:=\left(\E{\left(f_{\bm{\zeta}_t}(\bm{x}_\star)-\ell_{\bm{\zeta}_t}^*\right)^2}\right)^{1/2}.
\end{align}
Since each $f_{\bm{\zeta}}$ is lower bounded, $\sigma^2<\infty$. When the problem is \emph{interpolated}, i.e., there exists $\bm{x}_\star$ such that $f_{\bm{\zeta}}(\bm{x}_\star)=\inf_{\bm{x}}f_{\bm{\zeta}}(\bm{x})$ for all $\bm{\zeta}$, choosing $\ell_{\bm{\zeta}}^*=\inf_{\bm{x}}f_{\bm{\zeta}}(\bm{x})$ gives $\sigma^2=0$. The interpolation condition is satisfied in many over-parameterized settings \citep{ma2018power,zhang2021understanding}

\begin{theorem}
    \label{thm:safe}
    Consider the iterates of \ref{eq:sf} with the safeguarded step size \ref{eq:sf-sps-safe} and $c_t=1/(t+1)$. Let $f_{\bm{\zeta}}:\R^d\to\R$ be convex, differentiable, and $G$-Lipschitz for every $\bm{\zeta}$. With the initialization $\bm{z}_{-1}=\bm{x}_0$, the last iterate satisfies, for all $t\geq1$,
    \begin{align}
        \E{f(\bm{x}_t)-f(\bm{x}_\star)}+\frac{\beta}{1-\beta}\frac{1}{t+1}\sum_{k=0}^{t}\E{B_f(\bm{x}_k,\bm{y}_k)}\leq\frac{\sqrt{\max\{G^2,M\}}\,\norm{\bm{x}_0-\bm{x}_\star}}{\sqrt{t}}+\sqrt{\frac{\max\{G^2,M\}}{M}}\,\sigma^2,
        \label{eq:safe-rate}
    \end{align}
    where $B_f(\bm{x},\bm{y}):=f(\bm{x})-f(\bm{y})-\dotprod{\nabla f(\bm{y}),\bm{x}-\bm{y}}\geq0$ is the Bregman divergence of $f$.
\end{theorem}

\Cref{thm:safe} provides the same $\O(1/\sqrt{t})$ leading rate as \Cref{thm:oracle}, up to a neighborhood whose size is controlled by $\sigma^2$. The critical difference is that \ref{eq:sf-sps-safe} requires only a lower bound $\ell_{\bm{\zeta}_t}^*$ rather than the oracle value $f_{\bm{\zeta}_t}(\bm{x}_\star)$, making the method fully adaptive. There is also a difference in assumptions: \Cref{thm:oracle} requires $G$-Lipschitz continuity only locally on the ball $B$, while \Cref{thm:safe} requires it globally. This is because the safeguarded step size does not guarantee Fej\'er monotonicity of the $\bm{z}$ iterates, so the iterates cannot be confined to $B$ a priori.

Furthermore, the left-hand side of \eqref{eq:safe-rate} contains, in addition to the suboptimality gap, a non-negative Bregman divergence term weighted by $\beta/(1-\beta)$. Since the bound holds for the sum of two non-negative quantities, it implies in particular that $\E{f(\bm{x}_t)-f(\bm{x}_\star)}$ is bounded by the right-hand side. The Bregman term measures how far the evaluation sequence $\bm{x}_k$ departs from the gradient-location sequence $\bm{y}_k$, and its presence on the left-hand side suggests that these sequences stay close throughout training. When $\beta=0$, the coefficient $\beta/(1-\beta)$ vanishes, the Bregman term drops out, and \Cref{thm:safe} reduces to the \ref{eq:sps-safe} guarantee of \citet{oikonomou2025safeguarded}.

\section{Practical Variants}
\label{sec:practical}

The step size rules introduced in \Cref{sec:polyak-sf} and analyzed in \Cref{sec:convergence} are derived for the SGD-based Schedule-Free update. To make the method more broadly applicable in practice, we now discuss two extensions used in our experiments: a preconditioned Adam variant and an adaptive safeguard.

\subsection{Preconditioned Extension: Adam}
\label{sec:adam-polyak}

Adaptive gradient methods such as Adam \citep{Kingma2014AdamAM} and its decoupled weight-decay variant AdamW \citep{loshchilovdecoupled} have become the dominant optimizers for training deep networks. This line of adaptive preconditioning, building on AdaGrad \citep{duchi2011adaptive} and RMSProp \citep{tieleman2012lecture}, is especially critical for problems with heterogeneous curvature such as transformer training. The Schedule-Free framework has already been shown to combine effectively with this preconditioner, achieving state-of-the-art results across a wide range of tasks \citep{DYK+24}, which motivates extending our step size rules to this setting.

Let $\mD_t \succ 0$ be a symmetric positive definite preconditioner and let $\norm{\bm{v}}_{\mD_t^{-1}}^2 := \bm{v}^\top \mD_t^{-1} \bm{v}$. Replacing the gradient step with $\bm{z}_t = \bm{z}_{t-1} - \gamma_t \mD_t^{-1} \nabla f_{\bm{\zeta}_t}(\bm{y}_t)$ and expanding $\norm{\bm{z}_t - \bm{x}_\star}_{\mD_t}^2$ via the same argument as in \Cref{sec:polyak-sf-oracle} yields an upper bound identical to \eqref{eq:upper-bound} with $\norm{\nabla f_{\bm{\zeta}_t}(\bm{y}_t)}^2$ replaced by $\norm{\nabla f_{\bm{\zeta}_t}(\bm{y}_t)}_{\mD_t^{-1}}^2$. Minimizing over $\gamma_t \geq 0$, and applying the oracle-to-safeguarded modifications of \Cref{sec:polyak-sf-safe}, gives
\begin{align}
    \label{eq:sf-adam-sps-plus}
    \gamma_t &= \frac{\left[f_{\bm{\zeta}_t}(\bm{y}_t) - f_{\bm{\zeta}_t}(\bm{x}_\star) + \beta\dotprod{\nabla f_{\bm{\zeta}_t}(\bm{y}_t), \bm{z}_{t-1} - \bm{x}_t}\right]_+}{\norm{\nabla f_{\bm{\zeta}_t}(\bm{y}_t)}_{\mD_t^{-1}}^2},
    \tag{SF-Adam-SPS$_+$} \\[6pt]
    \label{eq:sf-adam-sps-safe}
    \gamma_t &= \frac{\left[f_{\bm{\zeta}_t}(\bm{y}_t) - \ell_{\bm{\zeta}_t}^* + \beta\dotprod{\nabla f_{\bm{\zeta}_t}(\bm{y}_t), \bm{z}_{t-1} - \bm{x}_t}\right]_+}{\max\{\norm{\nabla f_{\bm{\zeta}_t}(\bm{y}_t)}_{\mD_t^{-1}}^2, M\}}.
    \tag{SF-Adam-SPS$_{\text{safe}}$}
\end{align}
Setting $\mD_t = \mI$ recovers \ref{eq:sf-ps} and \ref{eq:sf-sps-safe}, respectively. The complete procedure, instantiated with the standard Adam diagonal preconditioner, is given in \Cref{alg:adam-sf-sps}.

\begin{remark}[Practical details]
    \label{rem:practical}
    Since $\bm{y}_t = (1-\beta)\bm{z}_{t-1} + \beta\bm{x}_t$ implies $\beta(\bm{z}_{t-1} - \bm{x}_t) = \bm{z}_{t-1} - \bm{y}_t$, the momentum correction appearing in all four step sizes (\ref{eq:sf-ps}, \ref{eq:sf-sps-safe}, \ref{eq:sf-adam-sps-plus}, \ref{eq:sf-adam-sps-safe}) satisfies
    \begin{equation*}
        \beta\dotprod{\nabla f_{\bm{\zeta}_t}(\bm{y}_t), \bm{z}_{t-1} - \bm{x}_t} = \dotprod{\nabla f_{\bm{\zeta}_t}(\bm{y}_t), \bm{z}_{t-1} - \bm{y}_t},
    \end{equation*}
    so the numerator can be computed directly from $\bm{z}_{t-1}$ and $\bm{y}_t$, avoiding the need to store $\bm{x}_t$.
\end{remark}

\subsection{Adaptive Safeguard via Exponential Moving Average}
\label{sec:ema-safeguard}

The safeguarded step sizes \ref{eq:sf-sps-safe} and \ref{eq:sf-adam-sps-safe} require choosing the constant $M>0$. While any positive value guarantees bounded step sizes, the choice affects practical performance: too large and the safeguard dominates; too small and it fails to prevent occasional large steps. Prior work on \ref{eq:sps-safe} for standard SGD observed similar sensitivity and proposed replacing the fixed $M$ with an exponential moving average of past squared gradient norms \citep{oikonomou2025safeguarded}. We adopt the same strategy for the Schedule-Free setting. Concretely, the safeguard is updated as
\begin{align}
    \label{eq:ema-safeguard}
    M_t = \beta_M \, M_{t-1} + (1-\beta_M)\,\norm{\nabla f_{\bm{\zeta}_t}(\bm{y}_t)}^2,
\end{align}
with $M_0 = \norm{\nabla f_{\bm{\zeta}_0}(\bm{y}_0)}^2$, and the fixed $M$ in \ref{eq:sf-sps-safe} or \ref{eq:sf-adam-sps-safe} is replaced by $M_t$. The smoothing parameter $\beta_M \in [0,1)$ controls the memory of the estimate; we recommend $\beta_M = 0.99$ as a robust default. This adaptive safeguard tracks the scale of the gradients throughout training without introducing additional hyperparameters that require careful tuning.

\section{Experiments}
\label{sec:experiments_main}

Our experiments evaluate the proposed Polyak step sizes across two settings, each targeting a different aspect of the method: black-box model distillation (\Cref{sec:modeldist}), where the teacher's loss approximates the oracle value and allows us to test \ref{eq:sf-ps}; and language model pretraining (\Cref{sec:langmod}), where we evaluate \ref{eq:sf-adam-sps-safe} at scale. Our main finding is that the proposed step sizes match or exceed tuned baselines in validation loss while exhibiting substantially greater robustness to hyperparameter choices.

\subsection{Model Distillation}
\label{sec:modeldist}

\begin{figure}
    \centering
    \begin{subfigure}[b]{0.49\linewidth}
        \centering
        \includegraphics[width=0.96\linewidth]{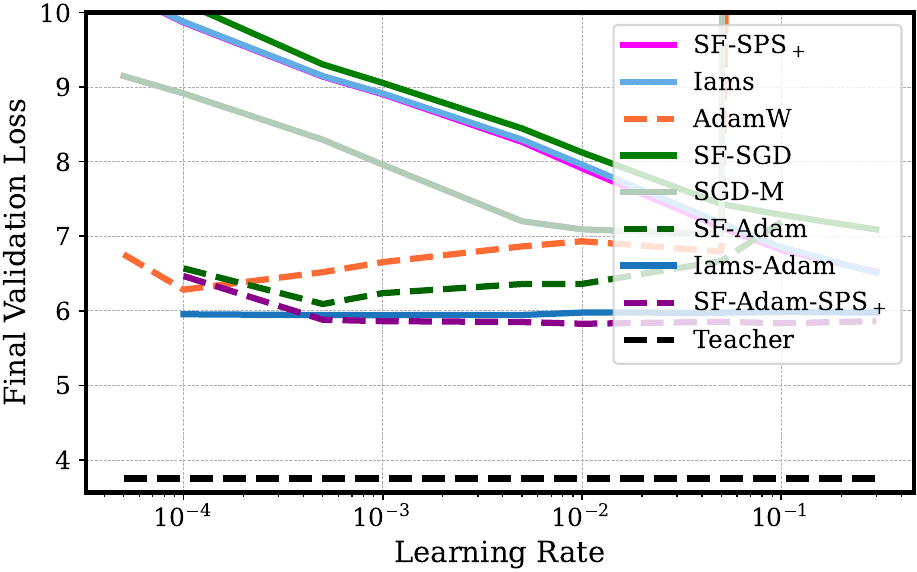}
        \caption{\texttt{tiny\_shakespeare}: LR sensitivity}
        \label{fig:shakespeare-lr}
    \end{subfigure}
    \begin{subfigure}[b]{0.49\linewidth}
        \centering
        \includegraphics[width=\linewidth]{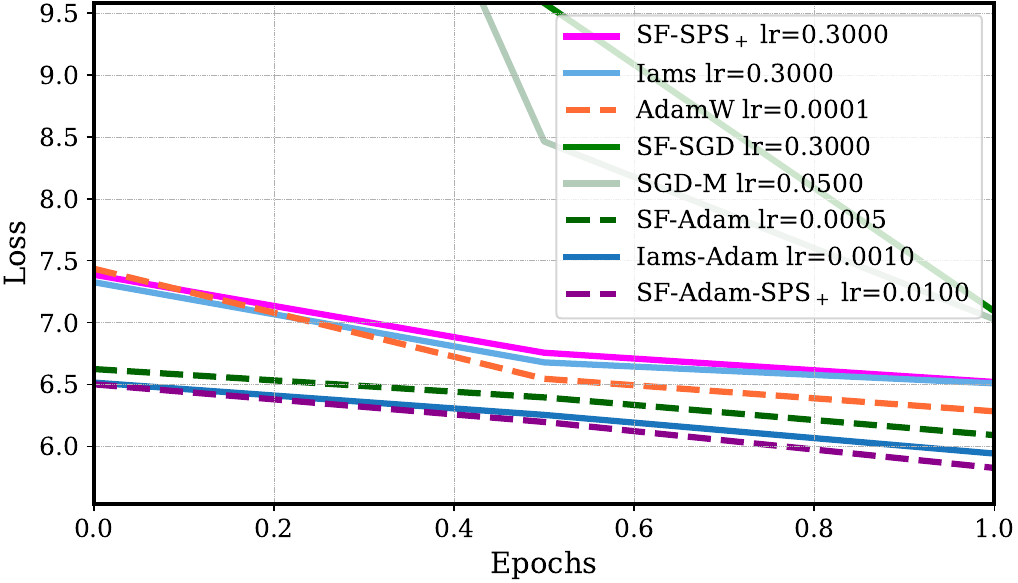}
        \caption{\texttt{tiny\_shakespeare}: validation loss}
        \label{fig:shakespeare}
    \end{subfigure}

    \caption{Black-box model distillation results. Distilling \texttt{gpt2-medium} into a 67.7M-parameter student on \texttt{tiny\_shakespeare}. Left: final validation loss as a function of the base learning rate. Right: validation loss curves for the best learning rate of each method.}
    \label{fig:distillation}
\end{figure}

The oracle step size \ref{eq:sf-ps} requires access to the per-sample optimal value $f_{\bm{\zeta}_t}(\bm{x}_\star)$, which is unavailable in general. Black-box model distillation offers a natural setting where this quantity can be approximated. Let $f^t_{\bm{\zeta}}$ and $f^s_{\bm{\zeta}}(\bm{x})$ denote the teacher's and student's loss on batch $\bm{\zeta}$, respectively. When the teacher has been trained to near-optimality on a large corpus and the student has sufficient capacity to match it, the teacher loss satisfies $f^t_{\bm{\zeta}} \approx f^s_{\bm{\zeta}}(\bm{x}_\star)$, where $\bm{x}_\star$ denotes the optimal student parameters. 

We test the Adam-based oracle variant \ref{eq:sf-adam-sps-plus} (see \Cref{alg:adam-sf-sps} for pseudocode). Our setup follows \citet{gower2025analysisidealizedstochasticpolyak}: we compare against SGD with momentum, Adam~\citep{Kingma2014AdamAM}, (Adam-)ScheduleFree~\citep{DYK+24}, and Iams(-Adam)~\citep{gower2025analysisidealizedstochasticpolyak}. For every baseline we sweep the base learning rate. For \ref{eq:sf-adam-sps-plus}, the step size is computed adaptively via $\gamma_t = \min\{\gamma_{\max},\, \tau_t\}$, where $\tau_t$ is the oracle Polyak step and $\gamma_{\max} > 0$ is a maximum step size that prevents excessively large updates when the teacher-loss approximation is imprecise. This cap is a standard safeguard in stochastic Polyak methods~\citep{LVLL21}; we sweep $\gamma_{\max}$ in the same manner as the base learning rate for the baselines. 

\paragraph{Distilling \texttt{tiny\_shakespeare}.} 
The teacher model employed was \texttt{gpt2-medium} (345 million parameters), a pre-trained transformer model from the Hugging Face library \citep{radford2019language}. We used a student model with 67.7 million parameters, see \Cref{tab:distillation_training_config} in \Cref{subsec:black-box-setting} for details. The results in \Cref{fig:shakespeare} show that our \ref{eq:sf-adam-sps-plus} achieves the best loss for a tuned learning rate $\gamma_{\max},$ and is  as robust as the Iams-Adam method  to the choice of learning rate.

\subsection{Language Modeling}
\label{sec:langmod}

We test \ref{eq:sf-adam-sps-safe} on pretraining GPT-2~\citep{radford2019language} style decoders on the FineWeb 10B dataset~\citep{penedo2024fineweb}, matching the architecture, batch size, number of training iterations, and schedules of~\citet{DYK+24}. Following \citet{DYK+24}, we replace the theoretical averaging weights $c_{t+1} = 1/(t+1)$ with the practical heuristic $c_{t+1} = \gamma_t^2 / \sum_{i=1}^t \gamma_i^2$ and use a warmup-stable schedule. We compare against AdamW and Schedule-Free AdamW, sweeping the base learning rate for both baselines over six values (see \Cref{tab:arch,tab:training,tab:optim} in \Cref{app:experiments} for full details). For \ref{eq:sf-adam-sps-safe}, the step size is computed adaptively and requires no learning rate sweep and introduce no additional per-step overhead, since the loss is already computed during the forward pass. We evaluate two safeguard strategies: a fixed constant $M=10$ and the EMA-based adaptive safeguard from \Cref{sec:ema-safeguard} with $\beta_M = 0.99$.

\Cref{fig:lr-stability-gptsmall} shows results for GPT-2 (small). In terms of training loss, \ref{eq:sf-adam-sps-safe} matches Schedule-Free AdamW across both $M$ settings, with both outperforming AdamW. In terms of validation loss, \ref{eq:sf-adam-sps-safe} achieves the best performance of all methods regardless of $M$ setting, suggesting better generalization. The LR sensitivity plot confirms that \ref{eq:sf-adam-sps-safe} requires no learning rate tuning, remaining flat across the full sweep while AdamW and Schedule-Free AdamW degrade sharply outside a narrow optimal range. We plot the adaptive Polyak step-sizes in \Cref{fig:polyak-steps-gptsmall} and plot the safeguard $M$ against preconditioned gradient norms in \Cref{fig:polyak-M-gptsmall}.

\begin{figure}
    \centering
    \begin{subfigure}[b]{0.45\linewidth}
        \includegraphics[width=\linewidth]{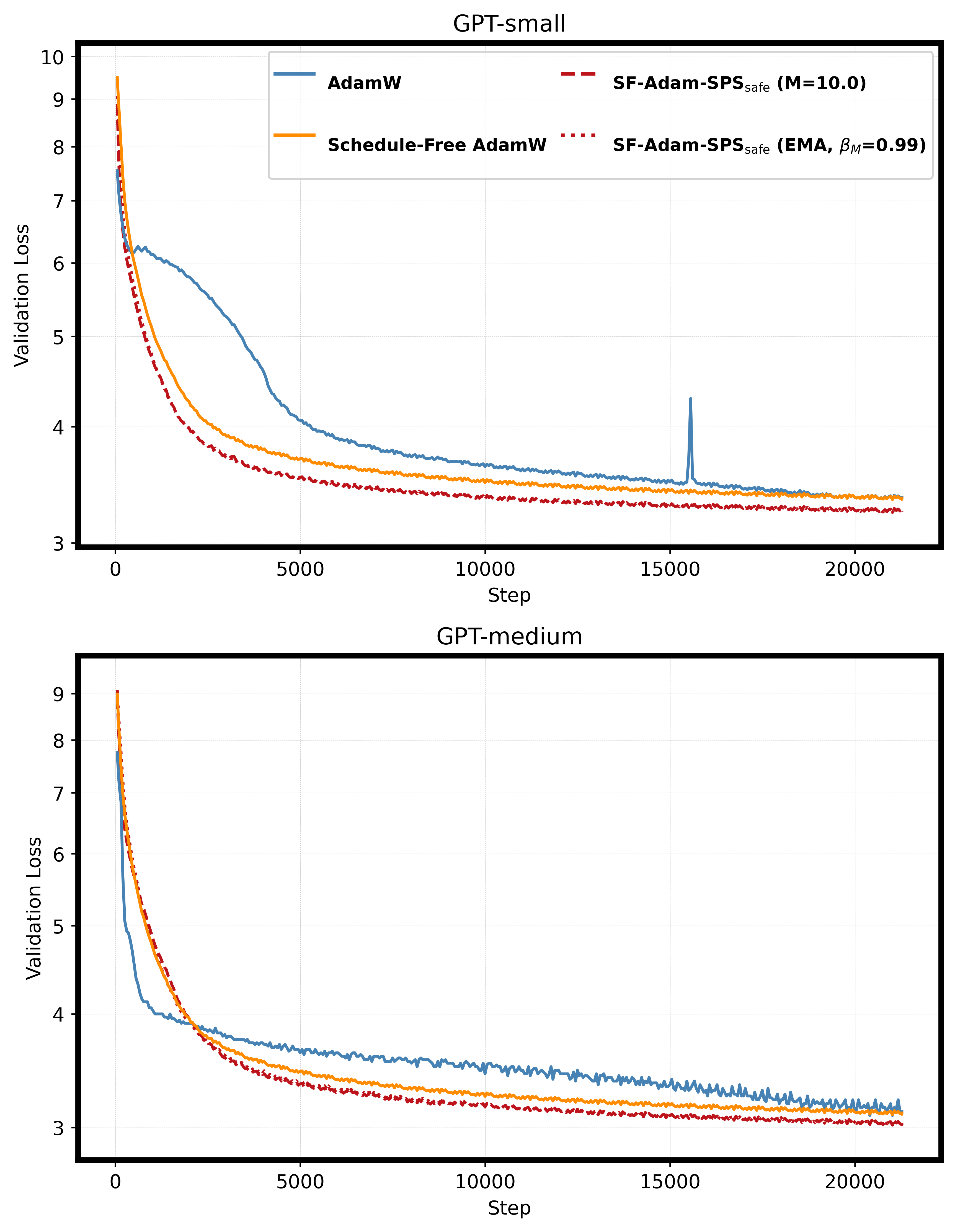}
    \end{subfigure}
    ~
    \begin{subfigure}[b]{0.45\linewidth}
        \includegraphics[width=\linewidth]{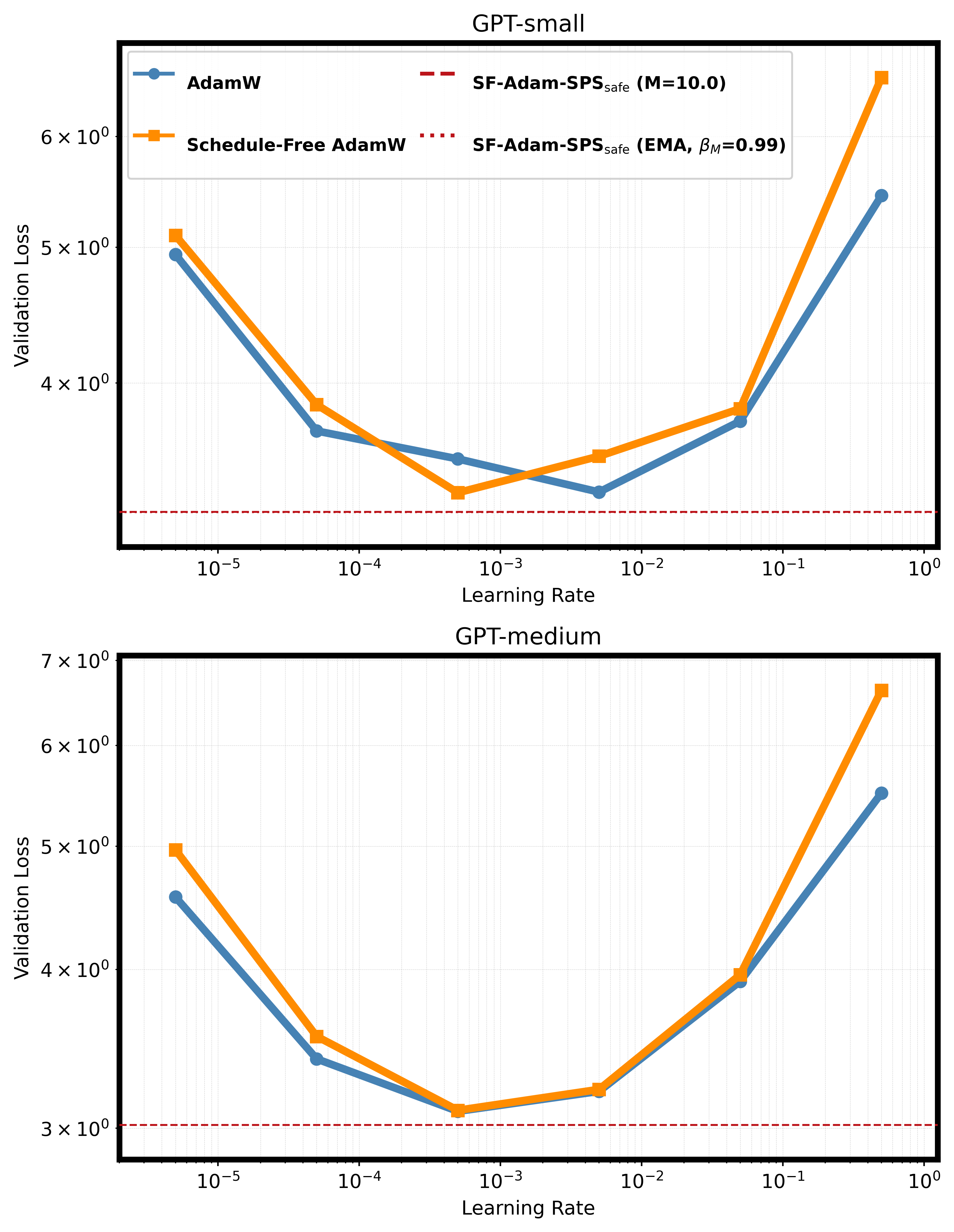}
    \end{subfigure}
    \caption{Validation loss (\textit{left}) and LR sensitivity (\textit{right}) for pretraining GPT2-small (124M parameters, \textit{top}) and GPT2-medium (335M parameters, \textit{bottom}) on the \texttt{fineweb10B} data set}
    \label{fig:lr-stability-gptsmall}
\end{figure}

\paragraph{Computational cost.} 
Beyond matching or exceeding tuned baselines in validation loss, the adaptive step size yields a substantial reduction in total compute. Finding the best learning rate for AdamW and Schedule-Free AdamW required sweeping over six values, consuming approximately 220 and 195 GPU-hours respectively on 4×H100s. By contrast, \ref{eq:sf-adam-sps-safe} computes its step size adaptively and requires no sweep: a single run completes in approximately 21 GPU-hours on 4×A100s: this would be faster on equivalent H100 hardware, since the H100 offers roughly 2–3× higher throughput for mixed-precision training. This makes the method particularly attractive in large-scale settings where learning rate tuning constitutes a significant fraction of the total training budget.

\section{Conclusion}
\label{sec:conc}

In this work, we derive Polyak-type step sizes for Schedule-Free SGD, yielding adaptive learning rates that require no knowledge of the Lipschitz constant, the initial distance to the optimum, or the stopping time. We provide last-iterate convergence guarantees for both an oracle and a safeguarded variant, achieving an $\mathcal{O}(1/\sqrt{t})$ rate for the oracle variant and the same rate up to a neighborhood for the safeguarded variant. These results offer a principled route to fully adaptive schedule-free optimization, reducing the tuning burden to a single safeguard constant. Future directions include extending the convergence analysis to smooth or strongly convex objectives to obtain faster rates, establishing guarantees for the Adam-based extension, and investigating tighter variance measures that could shrink or eliminate the convergence neighborhood beyond the interpolation regime.

\bibliographystyle{plainnat}
\bibliography{reference}

\newpage
\appendix

\newpage
\part*{Supplementary Material}

The Supplementary Material begins with a broader discussion of related work on Polyak-type step sizes (\Cref{app:related-polyak}), followed by the auxiliary lemmas and complete proofs of our convergence guarantees (\Cref{app:proofs}), and concludes with additional experimental results (\Cref{app:experiments}).

\tableofcontents

\newpage
\section{Related Work on Polyak Step Sizes}
\label{app:related-polyak}

The Polyak step size \citep{polyak1969minimization} computes the learning rate at each iteration from the current function value, the optimal value $f^\star$, and the gradient norm, removing the need to specify a base learning rate. While originally analyzed for convex Lipschitz objectives, \citet{HK19} broadened its theoretical foundations to cover smooth and strongly convex settings as well, and proposed a variant that estimates $f^\star$ online rather than requiring it as input. Extending these ideas to stochastic optimization, \citet{LVLL21} proposed SPS$_{\max}$ and established convergence in two settings: convex smooth objectives and, separately, convex Lipschitz objectives, with the latter requiring the interpolation condition. \citet{garrigos2023function} relaxed this by introducing the SPS$_+$ rule, which achieves the optimal rate in the Lipschitz setting without interpolation, at the cost of requiring per-sample optimal function values; they also proposed Function Value Learning to estimate these quantities in practice. To remove this dependence as well, \citet{oikonomou2025safeguarded} introduced a denominator safeguard that replaces the per-sample optimal values with any available lower bound, yielding convergence guarantees for non-smooth objectives under minimal assumptions.

Several works have refined and broadened SPS-type methods. \citet{orvieto2022dynamics} identified biases introduced by stochastic Polyak steps outside interpolation and proposed decreasing-step variants that converge exactly. \citet{orabona2025new} gave a complementary surrogate-loss perspective, interpreting the Polyak step size and several stochastic variants as gradient descent on non-negative surrogate functions, and also established negative results showing that some Polyak variants may converge only to a neighborhood outside favorable regimes such as interpolation. \citet{gower2022cutting} unified SPS-type rules through a cutting-plane perspective, introducing slack variants for non-interpolated settings. For composite objectives, \citet{schaipp2023stochastic} developed a proximal Polyak step size that handles non-smooth regularizers. \citet{d2023stochastic} generalized the framework to mirror descent in non-Euclidean geometries and \citet{SGDstruct} established Polyak step size guarantees for the broader class of quasar-convex functions. In the distributed setting, \citet{mukherjee2024locally} employed locally adaptive Polyak-type step sizes within a federated learning framework. Another line of work removes the need to know the optimal function value by estimating it online; for example, \citet{abdukhakimov2025polyak} proposed a parameter-free Polyak step size for deterministic gradient descent that maintains auxiliary iterates to estimate $f^\star$ during the run.

Closer to our setting are Polyak-type rules for momentum methods. \citet{schaipp2024momo} proposed MoMo, which builds momentum-based models of sampled losses to produce adaptive step sizes compatible with heavy-ball updates. \citet{oikonomou2024stochastic} designed Polyak step sizes for stochastic heavy-ball with guarantees both to neighborhoods and to exact minimizers via decreasing variants. \citet{pmlr-v202-wang23l} explored adaptive Polyak-type rules in momentum settings. \citet{orvieto2024NGN} connected Polyak-like adaptivity to second-order modeling by exploiting non-negative losses, writing them as squares, and applying stochastic Gauss-Newton or Levenberg-Marquardt steps that automatically warm up and decay with the loss landscape. \citet{gower2025analysisidealizedstochasticpolyak} analyzed an idealized stochastic Polyak method for momentum and applied it to black-box model distillation, which directly motivates our experimental setup.

\newpage
\section{Proofs of Convergence Guarantees}
\label{app:proofs}
 
This section contains the proofs of \Cref{thm:oracle,thm:safe}. We first collect the auxiliary lemmas used in both proofs, then present each theorem proof in turn.
 
\subsection{Auxiliary Lemmas}
\begin{lemma}[Extended Titu's Lemma] \label{lem:titu}
	For any random variable $X$ and positive-valued random variable $Y$, it holds
	\begin{equation}
		\E{\frac{(X)_+^2}{Y}} \ge \frac{\left(\E{X}\right)_+^2}{\E{Y}}. \label{eq:titu_expectation}
	\end{equation}
	In addition, for any numbers $a_0,\dotsc, a_k$ and positive numbers $b_0,\dotsc, b_k$, we have
	\begin{equation}
		\sum_{t=0}^k \frac{(a_t)_+^2}{b_t} \ge \frac{\left(\sum_{t=0}^k a_t\right)_+^2}{\sum_{t=0}^k b_t}. \label{eq:titu_numbers}
	\end{equation}
\end{lemma}
 
\begin{lemma}\label{lem:boundediterates}
If $f_{\bm{\zeta}}$ is convex for every $\bm{\zeta},$ and  we use the learning rate~\ref{eq:sf-ps} we have that   
\begin{align}
   \norm{\bm{z}_{t} - \bm{x}_{\star}}^2\
    &\le 
    \norm{\bm{z}_{t-1} - \bm{x}_{\star}}^2  -\frac{\big(f_{\bm{\zeta}_t}(\bm{y}_t) -f_{\bm{\zeta}_t}(\bm{x}_{\star}) +\beta \dotprod{\nabla f(\bm{y}_t,\bm{\zeta}_t),\bm{z}_{t-1} -\bm{x}_t } \big)_+^2}{\norm{\nabla f(\bm{y}_t,\bm{\zeta}_t)}^2 }.\label{eq:bounded-iterates}
\end{align}
As a consequence we also have that
$\|\bm{z}_t-\bm{x}_{\star}\|, \|\bm{x}_t-\bm{x}_{\star}\|$ and $\|\bm{y}_t-\bm{x}_{\star}\|$  are less than $\|\bm{x}_0 -\bm{x}_{\star}\|$.
Furthermore, taking expectation we have that
\begin{align}
   \EE{}{\norm{\bm{z}_{t} - \bm{x}_{\star}}^2 }
    & \le 
 \E{ \norm{\bm{z}_{t-1} - \bm{x}_{\star}}^2}  -\frac{\big(\E{f(\bm{y}_t) -f(\bm{x}_{\star}) +\beta \dotprod{\nabla f(\bm{y}_t),\bm{z}_{t-1} -\bm{x}_t } } \big)_+^2}{\E{\norm{\nabla f(\bm{y}_t,\bm{\zeta}_t)}^2} }.\label{eq:tesntoeisnt2E}
\end{align}
\end{lemma}
\begin{proof}
 Inserting~\ref{eq:sf-ps} into ~\eqref{eq:upper-bound}  gives the first result, which also shows that  $\norm{\bm{z}_{t} - \bm{x}_{\star}} \leq \norm{\bm{z}_{0} - \bm{x}_{\star}} = \norm{\bm{x}_{0} - \bm{x}_{\star}}$. Since $\bm{x}_{t+1}$ is a convex combination of $\bm{x}_t$ and $\bm{z}_t,$ we have that
 \[\|\bm{x}_{t+1} -\bm{x}_{\star}\| \leq (1-c_{t+1})\|\bm{x}_{t} -\bm{x}_{\star}\|  +c_{t+1}\|\bm{z}_{t} -\bm{x}_{\star}\| \]
 from which we can use induction to show $\norm{\bm{x}_{t} - \bm{x}_{\star}} \leq \norm{\bm{x}_{0} - \bm{x}_{\star}}.$ Furthermore, since $\bm{y}_t$ is a convex combination of $\bm{z}_{t-1}$ and $\bm{x}_t,$ it also follows by induction that $\| \bm{y}_t -\bm{x}_{\star}\| \leq \| \bm{x}_0 -\bm{x}_{\star}\|.$

    Taking conditional expectation over~\eqref{eq:bounded-iterates} given $\bm{x}_t$ and $\bm{z}_{t-1}$ and using Lemma~\ref{lem:titu} gives
    \begin{align}
        \mathbb{E}_t[\| \bm{z}_{t} - \bm{x}_{\star}\|^2] \le \| \bm{z}_{t-1} - \bm{x}_{\star} \|^2 - \frac{\big(f(\bm{y}_t) - f(\bm{x}_{\star}) + \beta \langle \nabla f(\bm{y}_t), \bm{z}_{t-1} - \bm{x}_t \rangle \big)_+^2}{\mathbb{E}_t[\|\nabla f(\bm{y}_t,\bm{\zeta}_t) \|^2]}. \label{eq:cond-expected-iterates}
    \end{align}
    Finally, taking total expectation over~\eqref{eq:cond-expected-iterates}, and using the law of total expectation and Lemma~\ref{lem:titu} again,  yields~\eqref{eq:tesntoeisnt2E}.
\end{proof}
 
Next we develop the Bregman viewpoint of this method. 
\begin{lemma}
Let $\lambda =\frac{\beta}{1-\beta} $. It follows that
\begin{align}
    f(\bm{y}_t) -f(\bm{x}_{\star}) +\beta\dotprod{\nabla f(\bm{y}_t),\bm{z}_{t-1} -\bm{x}_t } &= (1+\lambda)(f(\bm{y}_t) -f(\bm{x}_{\star}))  
  \nonumber\\
 & \quad -\lambda (f(\bm{x}_t) -f(\bm{x}_{\star})) \nonumber \\
    & \quad +\lambda  B_{f} (\bm{x}_t,\bm{y}_t),\label{eq:bregmanview}
\end{align}    
where $B_{f} (\bm{x}_t,\bm{y}_t)$ is the Bregman divergence of $f$ that is
\[ B_{f} (\bm{x},\bm{y}) := f (\bm{x}) - f(\bm{y}) -\dotprod{\nabla f(\bm{y}),\bm{x} -\bm{y} }. \]
\end{lemma}
\begin{proof}
Using  $ \bm{z}_{t-1}-\bm{x}_t = \frac{1}{1-\beta}(\bm{y}_t-\bm{x}_t) $ which follows from the $\bm{y}_t$ update in \ref{eq:sf} gives
    \begin{align*}
   f(\bm{y}_t) -f(\bm{x}_{\star}) +\beta\dotprod{\nabla f(\bm{y}_t),\bm{z}_{t-1} -\bm{x}_t } &=    f(\bm{y}_t) -f(\bm{x}_{\star}) -\frac{\beta}{1-\beta}\dotprod{\nabla f(\bm{y}_t),\bm{x}_t -\bm{y}_t } \\
   &=   (1+\lambda)(f(\bm{y}_t) -f(\bm{x}_{\star}))  -\lambda (f (\bm{x}_t) -f(\bm{x}_{\star}))\\
   & \quad +\lambda \big( f (\bm{x}_t) - f(\bm{y}_t) -\dotprod{\nabla f(\bm{y}_t),\bm{x}_t -\bm{y}_t } \big) .
\end{align*}
\end{proof}

\begin{lemma} \label{lem:telescoped}
    Let $c_t =1/(t+1)$. Initializing $\bm{z}_{-1} = \bm{x}_0$, it follows that
\begin{align}
   \EE{}{\norm{\bm{z}_{t} - \bm{x}_{\star}}^2 }
    & \le 
 \norm{\bm{x}_{0} - \bm{x}_{\star}}^2  \label{eq:recur-after-cond-exp-t-tele} \\
  & -\frac{\big((t+1)\E{f(\bm{x}_t) -f(\bm{x}_{\star})}  
+\lambda \sum_{k=0}^t \E{B_{f} (\bm{x}_k,\bm{y}_k)}\big)_+^2}{\sum_{k=0}^t\E{\norm{\nabla f(\bm{y}_k,\bm{\zeta}_k)}^2} }.\nonumber 
\end{align}
\end{lemma}

\begin{proof}
Using~\eqref{eq:bregmanview} in~\eqref{eq:tesntoeisnt2E} gives
\begin{align}
   \E{\norm{\bm{z}_{t} - \bm{x}_{\star}}^2 }
    & = 
  \E{\norm{\bm{z}_{t-1} - \bm{x}_{\star}}^2}  \nonumber\\
  &-\frac{\big(\E{(1+\lambda)(f(\bm{y}_t) -f(\bm{x}_{\star}))  
-\lambda (f (\bm{x}_t) -f(\bm{x}_{\star}))+\lambda  B_{f} (\bm{x}_t,\bm{y}_t)}\big)_+^2}{\E{\norm{\nabla f(\bm{y}_t,\bm{\zeta}_t)}^2} }.\label{eq:recur-after-cond-exp} 
\end{align}
Therefore, unrolling~\eqref{eq:recur-after-cond-exp} gives
\begin{align*}
   \E{\norm{\bm{z}_{t} - \bm{x}_{\star}}^2 }
    & \le 
  \norm{\bm{z}_{-1} - \bm{x}_{\star}}^2  -\sum_{k=0}^{t}\frac{\big(a_k\big)_+^2}{b_k},
\end{align*}
where we define $a_k\coloneqq \E{(1+\lambda)(f(\bm{y}_k) -f(\bm{x}_{\star}))  
-\lambda (f(\bm{x}_k) - f(\bm{x}_{\star}))+\lambda  B_{f} (\bm{x}_k,\bm{y}_k)}$ and $b_k = \E{\norm{\nabla f(\bm{y}_k,\bm{\zeta}_k)}^2}$. From Lemma~\ref{lem:titu}, we know that
\[\sum_{k=0}^{t} \frac{(a_k)_+^2}{b_k} \geq \frac{\big(\sum_{k=0}^{t}a_k\big)_+^2}{\sum_{k=0}^{t}b_k}.\]
 
Therefore, we have that
\begin{align}
   \EE{}{\norm{\bm{z}_{t} - \bm{x}_{\star}}^2 }
    & = 
 \norm{\bm{z}_{-1} - \bm{x}_{\star}}^2  \label{eq:recur-after-cond-exp-t} \\
  & -\frac{\big(\sum_{k=0}^t\E{(1+\lambda)(f(\bm{y}_k) - f(\bm{x}_{\star}))  
-\lambda (f(\bm{x}_k) - f(\bm{x}_{\star}))+\lambda  B_{f} (\bm{x}_k,\bm{y}_k)}\big)_+^2}{\sum_{k=0}^t\E{\norm{\nabla f(\bm{y}_k,\bm{\zeta}_k)}}^2 }.\nonumber 
\end{align}
 
To finish the proof of convergence, we need to write $\bm{y}_t$ as a combination of $\bm{x}_t$ and $\bm{x}_{t-1}$ so that we can telescope. To this end note that
\[ \bm{z}_{t-1} = \frac{1}{c_t}\bm{x}_t +\left(1 - \frac{1}{c_t} \right) \bm{x}_{t-1}.\]
Substituting this into the $\bm{y}_t$ update in \ref{eq:sf} gives
\begin{align*}
     \bm{y}_t &= (1-\beta) \left( \frac{1}{c_t}\bm{x}_t +\left(1 - \frac{1}{c_t} \right) \bm{x}_{t-1}\right) + \beta \bm{x}_t \\
     & =  \left( (1-\beta)\left(\frac{1}{c_t} -1\right) +1\right)\bm{x}_t -(1-\beta)\left( \frac{1}{c_t}-1 \right) \bm{x}_{t-1}.
     \end{align*}
Let $\rho_t := (1-\beta)\left( \frac{1}{c_t} -1\right).$ Isolating $\bm{x}_t$ in the above we have that it can be expressed as a convex combination between $\bm{y}_t$ and $\bm{x}_{t-1}$ given by 
\begin{equation}\label{eq:yttoxts}
    \bm{x}_t = \frac{1}{1+\rho_t}\bm{y}_t + \frac{\rho_t}{1+\rho_t} \bm{x}_{t-1}.
\end{equation} 
Using the convexity of $f$ we have that 
\begin{equation}\label{eq:convexytlosspre}
    f(\bm{x}_t) \leq  \frac{1}{1+\rho_t}f(\bm{y}_t) + \frac{\rho_t}{1+\rho_t} f(\bm{x}_{t-1}).
\end{equation} 
Re-arranging and isolating $f(\bm{y}_t)$ gives
\begin{equation}\label{eq:convexytloss} f(\bm{y}_t)  \geq (1+\rho_t)f(\bm{x}_t)  -\rho_t f(\bm{x}_{t-1}). \end{equation} 
Using the above we have that
\begin{align*}
    (1+\lambda)(f(\bm{y}_t) -f(\bm{x}_{\star}))  
-\lambda (f(\bm{x}_t) - f(\bm{x}_{\star})) &\geq 
 (1+\lambda)(1+\rho_t)(f(\bm{x}_t) -f(\bm{x}_{\star}))  \\
 & - (1+\lambda)\rho_t (f(\bm{x}_{t-1})-f(\bm{x}_{\star}))) \\
 & 
-\lambda (f(\bm{x}_t) - f(\bm{x}_{\star})) \\
&= (1+(1+\lambda)\rho_t )(f(\bm{x}_t) -f(\bm{x}_{\star})) \\
 & - (1+\lambda)\rho_t (f(\bm{x}_{t-1})-f(\bm{x}_{\star})))
\end{align*}
Substituting back $\rho_t := (1-\beta)\left( \frac{1}{c_t} -1\right) $ and 
$1+\lambda =\frac{1}{1-\beta} $ in the above and using that $c_t = 1/(t+\frac{1}{c_0})$ gives
\begin{align*}
    (1+\lambda)&(f(\bm{y}_t) -f(\bm{x}_{\star}))  
-\lambda (f (\bm{x}_t) -f(\bm{x}_{\star})) \\
&\geq \left(\frac{1}{c_t} \right)  (f(\bm{x}_t) - f(\bm{x}_{\star})) - \left(\frac{1}{c_t}-1 \right) (f(\bm{x}_{t-1})-f(\bm{x}_{\star})) \\
&= \left( t + \frac{1}{c_0}\right)(f(\bm{x}_t) - f(\bm{x}_{\star})) - \left( t-1 + \frac{1}{c_0}\right)(f(\bm{x}_{t-1})-f(\bm{x}_{\star})).
\end{align*}
Using the above we have that
\begin{align*}
    \sum_{k=0}^t&\big((1+\lambda)(f(\bm{y}_k) -f(\bm{x}_{\star}))  
-\lambda (f(\bm{x}_k) - f(\bm{x}_{\star}))\big) \\
& \geq f(\bm{x}_0) - f(\bm{x}_{\star}) +
  \sum_{k=1}^t\Bigg( \Big(k+\frac{1}{c_0}\Big)  (f(\bm{x}_k) -f(\bm{x}_{\star})) - \Big(k-1+\frac{1}{c_0}\Big)(f(\bm{x}_{k-1})-f(\bm{x}_{\star})) \Bigg) \\
  &=  f(\bm{x}_0) - f(\bm{x}_{\star}) + \left(t+\frac{1}{c_0}\right) (f(\bm{x}_t) -f(\bm{x}_{\star})) - \frac{1}{c_0} (f(\bm{x}_0) - f(\bm{x}_{\star})) \\
  &= (t+1)(f(\bm{x}_t) -f(\bm{x}_{\star})).
\end{align*} 
Inserting this in~\eqref{eq:recur-after-cond-exp-t}, together with the monotonicity of the positive part and the initialization that $\bm{z}_{-1}=\bm{x}_0$, gives
\begin{align}
   \EE{}{\norm{\bm{z}_{t} - \bm{x}_{\star}}^2 }
    & = 
 \norm{\bm{x}_{0} - \bm{x}_{\star}}^2  \label{eq:recur-after-cond-exp-t-tele-proof} \\
  & -\frac{\big((t+1)\E{f(\bm{x}_t) -f(\bm{x}_{\star})}  
+\lambda \sum_{k=0}^t \E{B_{f} (\bm{x}_k,\bm{y}_k)}\big)_+^2}{\sum_{k=0}^t\E{\norm{\nabla f(\bm{y}_k,\bm{\zeta}_k)}^2} }.\nonumber 
\end{align}
\end{proof}

\subsection{Proof of \Cref{thm:oracle} (Oracle Step Size)}

\begin{proof}
    Since $\|\bm{y}_k -\bm{x}_{\star}\| \leq \|\bm{x}_0-\bm{x}_{\star}\|$ we have that $\E{\norm{\nabla f(\bm{y}_k,\bm{\zeta}_k)}^2} \leq G^2$ and re-arranging~\eqref{eq:recur-after-cond-exp-t-tele}  gives
    \begin{align*}
     \big((t+1)\E{f(\bm{x}_t) - f(\bm{x}_{\star})}  
    +\lambda \sum_{k=0}^t \E{B_{f} (\bm{x}_k,\bm{y}_k)}\big)_+^2
        & \leq  
     G^2(t+1)(\norm{\bm{x}_{0} - \bm{x}_{\star}}^2  -  {\norm{\bm{z}_{t} - \bm{x}_{\star}}^2 } ) \\
     & \leq  G^2(t+1)\norm{\bm{x}_{0} - \bm{x}_{\star}}^2.
    \end{align*}
    Since the term on the left is always positive we can drop the positive part, taking square roots, and dividing through by $t+1$ gives
    \begin{align*}
       \E{f(\bm{x}_t) - f(\bm{x}_{\star})}  
    +\frac{\lambda}{ t+1} \sum_{k=0}^t \E{B_{f} (\bm{x}_k,\bm{y}_k)} \leq \frac{G\norm{\bm{x}_{0} - \bm{x}_{\star}}}{\sqrt{t+1}}.
    \end{align*}
    Inserting back $\lambda = \beta/(1-\beta)$ gives
    \begin{align*}
         \E{f(\bm{x}_t) - f(\bm{x}_{\star})}   \leq  \E{f(\bm{x}_t) - f(\bm{x}_{\star})} 
    +\frac{1}{ t+1} \frac{\beta}{1-\beta}\sum_{k=0}^t \E{B_{f} (\bm{x}_k,\bm{y}_k)} \leq \frac{G\norm{\bm{x}_{0} - \bm{x}_{\star}}}{\sqrt{t+1}}.
    \end{align*}
    Finally, we can drop the positive terms given by the Bregman divergences $\E{B_{f} (\bm{x}_k,\bm{y}_k)}$, giving the final desired result.
\end{proof}

\subsection{Proof of \Cref{thm:safe} (Safeguarded Step Size)}
\begin{proof}
    We have 
    \begin{align*}
        &\quad\|\bm{z}_t-\bm{x}_*\|^2-\|\bm{z}_{t-1}-\bm{x}_*\|^2\\
        &\leq-2\gamma_t\left[f_{\bm{\zeta}_t}(\bm{y}_t)-f_{\bm{\zeta}_t}(\bm{x}_*)+\beta\langle\nabla f_{\bm{\zeta}_t}(\bm{y}_t),\bm{z}_{t-1}-\bm{x}_t\rangle\right]+\gamma_t^2\|\nabla f_{\bm{\zeta}_t}(\bm{y}_t)\|^2\\
        &=-\frac{2[f_{\bm{\zeta}_t}(\bm{y}_t)-\ell_{\bm{\zeta}_t}^*+\beta\langle\nabla f_{\bm{\zeta}_t}(\bm{y}_t),\bm{z}_{t-1}-\bm{x}_t\rangle]_+[f_{\bm{\zeta}_t}(\bm{y}_t)-f_{\bm{\zeta}_t}(\bm{x}_*)+\beta\langle\nabla f_{\bm{\zeta}_t}(\bm{y}_t),\bm{z}_{t-1}-\bm{x}_t\rangle]}{\max\{\|\nabla f_{\bm{\zeta}_t}(\bm{y}_t)\|^2,M\}}\\
        &\quad+\frac{[f_{\bm{\zeta}_t}(\bm{y}_t)-\ell_{\bm{\zeta}_t}^*+\beta\langle\nabla f_{\bm{\zeta}_t}(\bm{y}_t),\bm{z}_{t-1}-\bm{x}_t\rangle]_+^2}{(\max\{\|\nabla f_{\bm{\zeta}_t}(\bm{y}_t)\|^2,M\})^2}\|\nabla f_{\bm{\zeta}_t}(\bm{y}_t)\|^2,\quad q=f_{\bm{\zeta}_t}(\bm{y}_t)+\beta\langle\nabla f_{\bm{\zeta}_t}(\bm{y}_t),\bm{z}_{t-1}-\bm{x}_t\rangle\\
        &=\frac{-2(q-\ell_{\bm{\zeta}_t}^*)_+\cdot(q-f_{\bm{\zeta}_t}(\bm{x}_*))}{\max\{\|\nabla f_{\bm{\zeta}_t}(\bm{y}_t)\|^2,M\}}+\frac{(q-\ell_{\bm{\zeta}_t}^*)_+^2}{\max\{\|\nabla f_{\bm{\zeta}_t}(\bm{y}_t)\|^2,M\}}\cdot\frac{\|\nabla f_{\bm{\zeta}_t}(\bm{y}_t)\|^2}{\max\{\|\nabla f_{\bm{\zeta}_t}(\bm{y}_t)\|^2,M\}}\\
        &\leq\frac{-2(q-\ell_{\bm{\zeta}_t}^*)_+\cdot(q-f_{\bm{\zeta}_t}(\bm{x}_*))}{\max\{\|\nabla f_{\bm{\zeta}_t}(\bm{y}_t)\|^2,M\}}+\frac{(q-\ell_{\bm{\zeta}_t}^*)_+^2}{\max\{\|\nabla f_{\bm{\zeta}_t}(\bm{y}_t)\|^2,M\}}\\
        &=\frac{-2(q-\ell_{\bm{\zeta}_t}^*)_+\cdot(q-f_{\bm{\zeta}_t}(\bm{x}_*))+(q-\ell_{\bm{\zeta}_t}^*)_+^2}{\max\{\|\nabla f_{\bm{\zeta}_t}(\bm{y}_t)\|^2,M\}}\\
        &\overset{(\star)}{\leq}\frac{(f_{\bm{\zeta}_t}(\bm{x}_*)-\ell_{\bm{\zeta}_t}^*)^2-(q-f_{\bm{\zeta}_t}(\bm{x}_*))_+^2}{\max\{\|\nabla f_{\bm{\zeta}_t}(\bm{y}_t)\|^2,M\}}\\
        &=-\frac{(f_{\bm{\zeta}_t}(\bm{y}_t)-f_{\bm{\zeta}_t}(\bm{x}_*)+\beta\langle\nabla f_{\bm{\zeta}_t}(\bm{y}_t),\bm{z}_{t-1}-\bm{x}_t\rangle)_+^2}{\max\{\|\nabla f_{\bm{\zeta}_t}(\bm{y}_t)\|^2,M\}}+\frac{[f_{\bm{\zeta}_t}(\bm{x}_*)-\ell_{\bm{\zeta}_t}^*]^2}{\max\{\|\nabla f_{\bm{\zeta}_t}(\bm{y}_t)\|^2,M\}}\\
        &\overset{(\star\star)}{\leq}-\frac{(f_{\bm{\zeta}_t}(\bm{y}_t)-f_{\bm{\zeta}_t}(\bm{x}_*)+\beta\langle\nabla f_{\bm{\zeta}_t}(\bm{y}_t),\bm{z}_{t-1}-\bm{x}_t\rangle)_+^2}{\max\{G^2,M\}}+\frac{[f_{\bm{\zeta}_t}(\bm{x}_*)-\ell_{\bm{\zeta}_t}^*]^2}{M}\\
    \end{align*}
    Let's explain inequality $(\star)$: Note that $\ell_{\bm{\zeta}_t}^*\leq f_{\bm{\zeta}_t}(\bm{x}_*)$ so 
    $q-\ell_{\bm{\zeta}_t}^*\geq q-f_{\bm{\zeta}_t}(\bm{x}_*)$. Hence if $q-\ell_{\bm{\zeta}_t}^*\leq0$ inequality $(\star)$ 
    reduces to the obvious $0\leq[f_{\bm{\zeta}_t}(\bm{x}_*)-\ell_{\bm{\zeta}_t}^*]^2$. Now assume that 
    $q-\ell_{\bm{\zeta}_t}^*>0$. Then 
    \begin{align*}
        -2(q-\ell_{\bm{\zeta}_t}^*)_+\cdot(q-f_{\bm{\zeta}_t}(\bm{x}_*))+(q-\ell_{\bm{\zeta}_t}^*)_+^2
        &=-2(q-\ell_{\bm{\zeta}_t}^*)(q-f_{\bm{\zeta}_t}(\bm{x}_*))+(q-\ell_{\bm{\zeta}_t}^*)^2\\
        &=(q-\ell_{\bm{\zeta}_t}^*-(q-f_{\bm{\zeta}_t}(\bm{x}_*)))^2-(q-f_{\bm{\zeta}_t}(\bm{x}_*))^2\\
        &=(f_{\bm{\zeta}_t}(\bm{x}_*)-\ell_{\bm{\zeta}_t}^*)^2-(q-f_{\bm{\zeta}_t}(\bm{x}_*))^2\\
        &\leq(f_{\bm{\zeta}_t}(\bm{x}_*)-\ell_{\bm{\zeta}_t}^*)^2-(q-f_{\bm{\zeta}_t}(\bm{x}_*))_+^2.
    \end{align*}
    Inequality $(\star\star)$ follows from $\max\{\|\nabla f_{\bm{\zeta}_t}(\bm{y}_t)\|^2,M\}\geq M$ 
    and $\max\{\|\nabla f_{\bm{\zeta}_t}(\bm{y}_t)\|^2,M\}\leq\max\{G^2,M\}$ because $f_{\bm{\zeta}_t}$ is 
    $G$-Lipschitz. 
    
    Take (double) expectation, use Titu lemma, tower property and 
    Bregman viewpoint to get 
    \begin{align*}
        &\E{\|\bm{z}_t-\bm{x}_*\|^2}-\E{\|\bm{z}_{t-1}-\bm{x}_*\|^2}\\
        &\leq-\frac{\E{(f(\bm{y}_t)-f(\bm{x}_*)+\beta\langle\nabla f(\bm{y}_t),\bm{z}_{t-1}-\bm{x}_t\rangle)_+^2}}
        {\max\{G^2,M\}}+\frac{\sigma^4}{M}\\
        &=-\frac{\E{((1+\lambda)(f(\bm{y}_t)-f(\bm{x}_*))-\lambda(f(\bm{x}_t)-f(\bm{x}_*))+\lambda B_f(\bm{x}_t,\bm{y}_t))_+^2}}
        {\max\{G^2,M\}}+\frac{\sigma^4}{M},
    \end{align*}
    where $\lambda=\beta/(1-\beta)$, so 
    \begin{align*}
        &\E{((1+\lambda)(f(\bm{y}_t)-f(\bm{x}_*))-\lambda(f(\bm{x}_t)-f(\bm{x}_*))+\lambda B_f(\bm{x}_t,\bm{y}_t))_+^2}\\
        &\leq \max\{G^2,M\}\E{\|\bm{z}_{t-1}-\bm{x}_*\|^2}-\max\{G^2,M\}\E{\|\bm{z}_t-\bm{x}_*\|^2}+
        \frac{\max\{G^2,M\}}{M}\sigma^4.
    \end{align*}
    Now let $\Delta_t=(1+\lambda)(f(\bm{y}_t)-f(\bm{x}_*))-\lambda(f(\bm{x}_t)-f(\bm{x}_*))+\lambda B_f(\bm{x}_t,\bm{y}_t)$, sum 
    for $t=0,\dots,T-1$ and use Jensen to get 
    \begin{align*}
        &\quad\frac{\max\{G^2,M\}\|\bm{x}_0-\bm{x}_*\|^2}{T}+\frac{\max\{G^2,M\}}{M}\sigma^4\\
        &\geq\frac{1}{T}\sum_{t=0}^{T-1}\E{(\Delta_t)_+^2}\\
        &\geq\left(\frac{1}{T}\sum_{t=0}^{T-1}\E{\Delta_t}\right)_+^2,
    \end{align*}
    which means that 
    \begin{align*}
        \left(\frac{1}{T}\sum_{t=0}^{T-1}\E{\Delta_t}\right)_+
        &\leq\sqrt{\frac{\max\{G^2,M\}\|\bm{x}_0-\bm{x}_*\|^2}{T}+\frac{\max\{G^2,M\}}{M}\sigma^4}\\
        &\leq\frac{\sqrt{\max\{G^2,M\}}\|\bm{x}_0-\bm{x}_*\|}{\sqrt{T}}+\sqrt{\frac{\max
        \{G^2,M\}}{M}}\sigma^2.
    \end{align*}
    Now by algebraic manipulations we have $\bm{x}_t=\frac{1}{1+\rho_t}\bm{y}_t+\frac{\rho_t}
    {1+\rho_t}\bm{x}_{t-1}$, using convexity and solving for $f(\bm{y}_t)$ we get 
    \begin{align*}
        f(\bm{y}_t)\geq(1+\rho_t)f(\bm{x}_t)-\rho_tf(\bm{x}_{t-1}),
    \end{align*}
    where $\rho_t=(1-\beta)(\frac{1}{c_t}-1)$. Now we have 
    \begin{align*}
        &(1+\lambda)[f(\bm{y}_t)-f(\bm{x}_*)]-\lambda[f(\bm{x}_t)-f(\bm{x}_*)]+\lambda B_f(\bm{x}_t,\bm{y}_t)\\
        &\geq(t+1)[f(\bm{x}_t)-f(\bm{x}_*)]-t[f(\bm{x}_{t-1})-f(\bm{x}_*)]+\lambda B_f(\bm{x}_t,\bm{y}_t),
    \end{align*}
    thus 
    \begin{align*}
        \frac{1}{T}\sum_{t=0}^{T-1}\E{\Delta_t}&=\frac{1}{T}\sum_{t=0}^{T-1}\E{
        (1+\lambda)(f(\bm{y}_t)-f(\bm{x}_*))-\lambda(f(\bm{x}_t)-f(\bm{x}_*))+\lambda B_f(\bm{x}_t,\bm{y}_t)}\\
        &\geq\frac{1}{T}\E{\sum_{t=0}^{T-1}(t+1)(f(\bm{x}_t)-f(\bm{x}_*))-t(f(\bm{x}_{t-1})-
        f(\bm{x}_*))+\lambda B_f(\bm{x}_t,\bm{y}_t)}\\
        &=\E{f(\bm{x}_{T-1})-f(\bm{x}_*)}+\frac{\lambda}{T}\sum_{t=0}^{T-1}\E{B_f(\bm{x}_t,\bm{y}_t)}\\
        &\geq0.
    \end{align*}
    Hence 
    \begin{align*}
        \left(\frac{1}{T}\sum_{t=0}^{T-1}\E{\Delta_t}\right)_+\geq\E{f(\bm{x}_{T-1})-f(\bm{x}_*)}
        +\frac{\lambda}{T}\sum_{t=0}^{T-1}\E{B_f(\bm{x}_t,\bm{y}_t)},
    \end{align*}
    thus 
    \begin{align*}
        \E{f(\bm{x}_{T-1})-f(\bm{x}_*)}+\frac{\beta}{1-\beta}\frac{1}{T}\sum_{t=0}^{T-1}\E{B_f(\bm{x}_t,\bm{y}_t)}
        &\leq\frac{\sqrt{\max\{G^2,M\}}\|\bm{x}_0-\bm{x}_*\|}{\sqrt{T}}+\sqrt{\frac{\max
        \{G^2,M\}}{M}}\sigma^2,
    \end{align*}
    as wanted. 
\end{proof}

\section{Additional Experimental Results}
\label{app:experiments}

\vspace{-3mm}

\subsection{Pseudocodes}

For completeness, we provide the pseudocodes for all the proposed algorithms and step sizes used in our experiments. 

\begin{algorithm}[h]
\begin{algorithmic}[1]
    \caption{\ref{eq:sf} with \ref{eq:sf-ps} or \ref{eq:sf-sps-safe}}
    \State \textbf{Input:} $\bm{z}_{-1}=\bm{x}_0\in\R^d$, $\beta\in[0,1)$, lower bounds $\ell_{\bm{\zeta}_t}^*$, safeguard $M>0$
    \For{$t=0$ \textbf{to} $T-1$}
        \State $\bm{y}_t=(1-\beta)\bm{z}_{t-1}+\beta\bm{x}_t$
        \State Sample $\bm{\zeta}_t$; compute $\nabla f_{\bm{\zeta}_t}(\bm{y}_t)$
        \State \textbf{Option I:} $\displaystyle\gamma_t=\frac{\left[f_{\bm{\zeta}_t}(\bm{y}_t)-f_{\bm{\zeta}_t}(\bm{x}_\star)+\beta\dotprod{\nabla f_{\bm{\zeta}_t}(\bm{y}_t),\bm{z}_{t-1}-\bm{x}_t}\right]_+}{\norm{\nabla f_{\bm{\zeta}_t}(\bm{y}_t)}^2}$ \hfill $\triangleright$ \ref{eq:sf-ps}
        \Statex \hspace{\algorithmicindent}\textbf{or}
        \State  \textbf{Option II:} $\displaystyle\gamma_t=\frac{\left[f_{\bm{\zeta}_t}(\bm{y}_t)-\ell_{\bm{\zeta}_t}^*+\beta\dotprod{\nabla f_{\bm{\zeta}_t}(\bm{y}_t),\bm{z}_{t-1}-\bm{x}_t}\right]_+}{\max\{\norm{\nabla f_{\bm{\zeta}_t}(\bm{y}_t)}^2,M\}}$ \hfill $\triangleright$ \ref{eq:sf-sps-safe}
        \State $\bm{z}_t=\bm{z}_{t-1}-\gamma_t\nabla f_{\bm{\zeta}_t}(\bm{y}_t)$
        \State $\bm{x}_{t+1}=(1-c_{t+1})\bm{x}_t+c_{t+1}\bm{z}_t$, \quad $c_{t+1}=\frac{1}{t+1}$
    \EndFor
    \State \textbf{Return:} $\bm{x}_T$
\end{algorithmic}
\end{algorithm}

\begin{algorithm}[h]
\begin{algorithmic}[1]
    \caption{Schedule-Free Adam with \ref{eq:sf-adam-sps-plus} or \ref{eq:sf-adam-sps-safe}}
    \label{alg:adam-sf-sps}
    \State \textbf{Input:} $\bm{z}_{-1}=\bm{x}_0\in\R^d$, $\beta\in[0,1)$, $\beta_2\in[0,1)$, $\epsilon>0$, lower bounds $\ell_{\bm{\zeta}_t}^*$, safeguard $M>0$
    \State $\bm{v}_{-1}=\bm{0}$
    \For{$t=0$ \textbf{to} $T-1$}
        \State $\bm{y}_t=(1-\beta)\bm{z}_{t-1}+\beta\bm{x}_t$
        \State Sample $\bm{\zeta}_t$; compute $\nabla f_{\bm{\zeta}_t}(\bm{y}_t)$
        \State $\bm{v}_t=\beta_2\bm{v}_{t-1}+(1-\beta_2)\left(\nabla f_{\bm{\zeta}_t}(\bm{y}_t)\right)^2$ \Comment{Element-wise square}
        \State $\mD_t=\operatorname{diag}\!\left(\sqrt{\bm{v}_t/(1-\beta_2^{t+1})}+\epsilon\right)$ \Comment{Adam preconditioner}
        \State \textbf{Option I:} $\displaystyle\gamma_t=\frac{\left[f_{\bm{\zeta}_t}(\bm{y}_t)-f_{\bm{\zeta}_t}(\bm{x}_\star)+\beta\dotprod{\nabla f_{\bm{\zeta}_t}(\bm{y}_t),\bm{z}_{t-1}-\bm{x}_t}\right]_+}{\norm{\nabla f_{\bm{\zeta}_t}(\bm{y}_t)}_{\mD_t^{-1}}^2}$ \hfill $\triangleright$ \ref{eq:sf-adam-sps-plus}
        \Statex \hspace{\algorithmicindent}\textbf{or}
        \State \textbf{Option II:} $\displaystyle\gamma_t=\frac{\left[f_{\bm{\zeta}_t}(\bm{y}_t)-\ell_{\bm{\zeta}_t}^*+\beta\dotprod{\nabla f_{\bm{\zeta}_t}(\bm{y}_t),\bm{z}_{t-1}-\bm{x}_t}\right]_+}{\max\{\norm{\nabla f_{\bm{\zeta}_t}(\bm{y}_t)}_{\mD_t^{-1}}^2,M\}}$ \hfill $\triangleright$ \ref{eq:sf-adam-sps-safe}
        \State $\bm{z}_t=\bm{z}_{t-1}-\gamma_t\mD_t^{-1}\nabla f_{\bm{\zeta}_t}(\bm{y}_t)$
        \State $\bm{x}_{t+1}=(1-c_{t+1})\bm{x}_t+c_{t+1}\bm{z}_t$, \quad $c_{t+1}=\frac{1}{t+1}$
    \EndFor
    \State \textbf{Return:} $\bm{x}_T$
\end{algorithmic}
\end{algorithm}

\vspace{-3mm}
\subsection{Black-Box Distillation details}\label{subsec:black-box-setting}

\begin{table}[h]
\centering
\begin{tabular}{ll}
\toprule
\textbf{Experiment} & \textbf{\texttt{tiny\_shakespeare}} \\
\midrule
Teacher model & \texttt{gpt2-medium} \\
Student hidden size & 768 \\
Student transformer layers & 4 \\
Student attention heads & 8 \\
Student vocabulary size & 50257 \\
Batch size & 4 \\
Context length & 512 tokens \\
Tokens per training step & 4096 \\
Learning rate schedule & constant \\
$(\beta_1,\beta_2)$ & (0.9, 0.999) \\
\bottomrule
\end{tabular}
\caption{Model configuration and training setup for distillation on \texttt{tiny\_shakespeare}.}
\label{tab:distillation_training_config}
\end{table}

Mixed precision training was enabled using \texttt{bfloat16} for efficiency.   The student model utilized flash attention~\citep{dao2022flashattention}.  The full details of the architecture and hyperparameters we used are in~\Cref{tab:distillation_training_config}.

\subsection{Pretraining GPT-2 Decoder Experiment Details}

We follow the GPT-2 architecture~\cite{radford2019language} and train on FineWeb 10B~\cite{penedo2024fineweb}. As in \cite{DYK+24}, gradient clipping is disabled for Schedule-Free and \ref{eq:sf-adam-sps-safe} runs.

\vspace{-3mm}
\paragraph{Architecture.}
\Cref{tab:arch} summarizes the model configurations. All models use flash attention, a vocabulary of 50{,}257 tokens, and a context length of 1{,}024.

\begin{table}[H]
\vspace{-2mm}
\centering
\caption{Model architectures and hardware.}
\label{tab:arch}
\begin{tabular}{lrrr}
\toprule
\textbf{Hyperparameter} & \textbf{GPT-small} & \textbf{GPT-medium}\\
\midrule
Parameters        & 124M   & 355M \\
Num.\ embeddings  & 768    & 1{,}024  \\
Num.\ layers      & 12     & 24 \\
Num.\ heads       & 12     & 16  \\
GPUs              & $4\times$H100 & $4\times$H100 \\
\bottomrule
\end{tabular}
\vspace{-4mm}
\end{table}

\paragraph{Training.}
\Cref{tab:training} lists the common training hyperparameters shared across all models.

\begin{table}[H]
\vspace{-2mm}
\centering
\caption{Common training hyperparameters.}
\label{tab:training}
\begin{tabular}{ll}
\toprule
\textbf{Hyperparameter} & \textbf{Value} \\
\midrule
Dataset                   & FineWeb 10B \\
Total tokens              & 10.25B \\
Context length            & 1{,}024 \\
Global batch size         & 480 sequences (491{,}520 tokens) \\
Per-GPU batch size        & 12 \\
Gradient accum.\ steps    & 10 \\
Optimizer steps           & 20{,}856 \\
Weight decay              & 0.1 \\
Dropout                   & 0.0 \\
Gradient clipping         & 1.0 (AdamW only) \\
\bottomrule
\end{tabular}
\vspace{-4mm}
\end{table}

\paragraph{Optimizers.}

\Cref{tab:optim} lists the optimizer-specific hyperparameters. For AdamW and Schedule-Free Adam we sweep the learning rate over six values and report the best validation loss. The Polyak step size is determined adaptively and requires no LR sweep, we therefore plot the final training and validation losses as a constant horizontal line to compare against the performance of the baselines for all base learning rate settings.

\begin{table}[H]
\vspace{-2mm}
\centering
\caption{Optimizer hyperparameters.}
\label{tab:optim}
\resizebox{\textwidth}{!}{%
\begin{tabular}{lllll}
\toprule
\textbf{Hyperparameter}
  & \textbf{AdamW}
  & \textbf{SF-Adam}
  & \textbf{\ref{eq:sf-adam-sps-safe} (const $M$)}
  & \textbf{\ref{eq:sf-adam-sps-safe} (EMA $M$)} \\
\midrule
$\beta_1$            & 0.9  & 0.9  & 0.9  & 0.9  \\
$\beta_2$            & 0.95 & 0.98 & 0.98 & 0.98 \\
Warmup steps         & --- & 2{,}000 & 2{,}000 & 2{,}000 \\
LR schedule          & linear decay & constant & constant & constant \\
LR sweep             & $\{5\text{e-}6,\;5\text{e-}5,\;5\text{e-}4,\;5\text{e-}3,\;5\text{e-}2,\;5\text{e-}1\}$
                     & $\{5\text{e-}6,\;5\text{e-}5,\;5\text{e-}4,\;5\text{e-}3,\;5\text{e-}2,\;5\text{e-}1\}$
                     & --- & --- \\
$M$                  & ---  & ---  & $10$ & --- \\
$\beta_M$ (EMA coeff.)  & ---  & ---  & --- & 0.99 \\
$\ell_{\zeta_t}^*$  & ---  & ---  & 0.0 & 0.0 \\
\bottomrule
\end{tabular}}
\vspace{-4mm}
\end{table}

\paragraph{Visualization of Polyak Step Sizes}

\Cref{fig:polyak-steps-gptsmall} shows the adaptive step-sizes computed by \ref{eq:sf-adam-sps-safe} over the course of training, compared against the fixed learning rates of AdamW and Schedule-Free AdamW tuned for best final validation loss. The Polyak step-sizes are large early in training when the loss gap is large, and decay naturally as the model converges — recovering a schedule-like decay without any explicit schedule.

\begin{figure}[H]
    \centering
    \begin{subfigure}[b]{0.48\linewidth}
        \includegraphics[width=\linewidth]{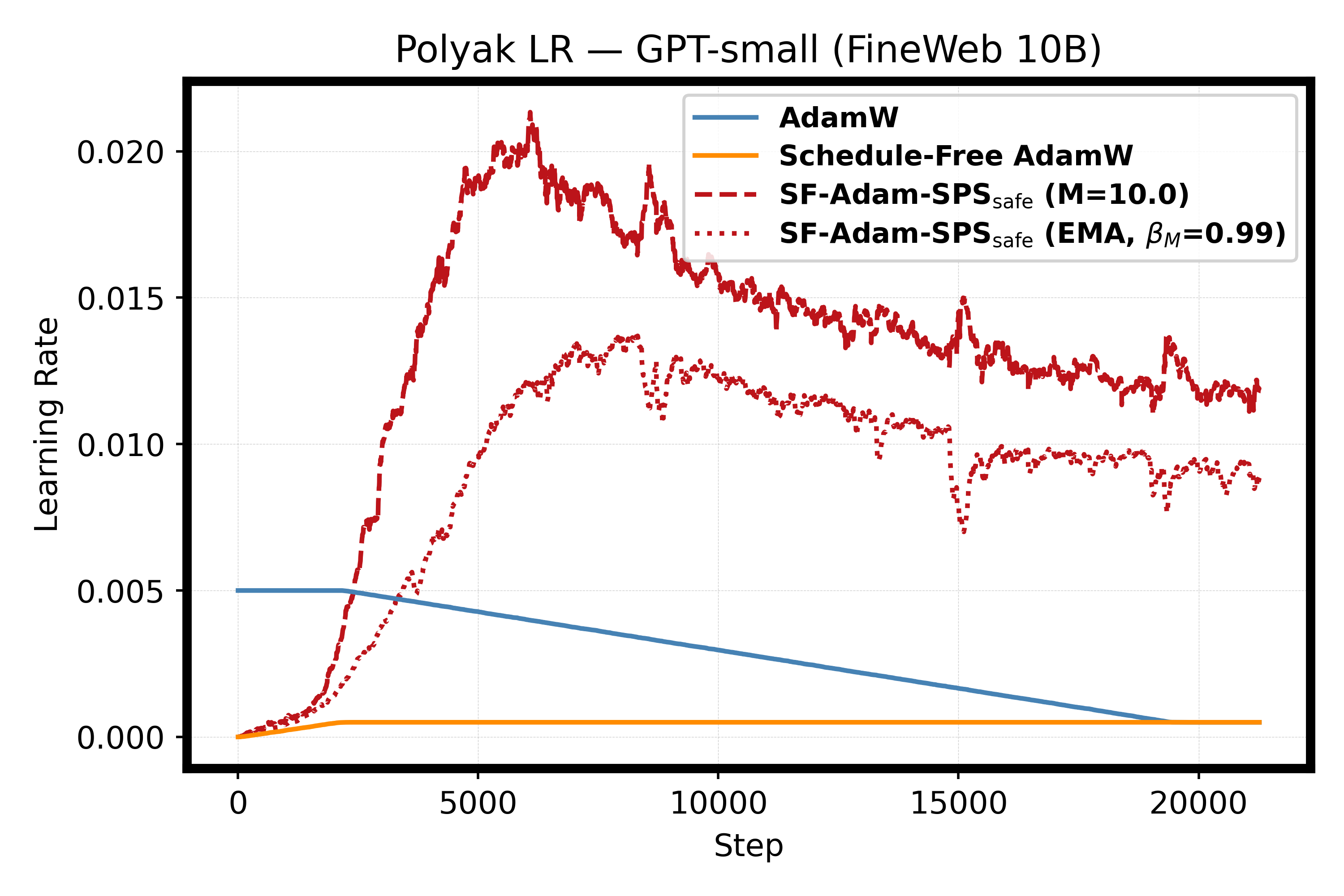}
        \caption{GPT2-small}
        \label{fig:polyak-steps-gptsmall}
    \end{subfigure}
    \hfill
    \begin{subfigure}[b]{0.48\linewidth}
        \includegraphics[width=\linewidth]{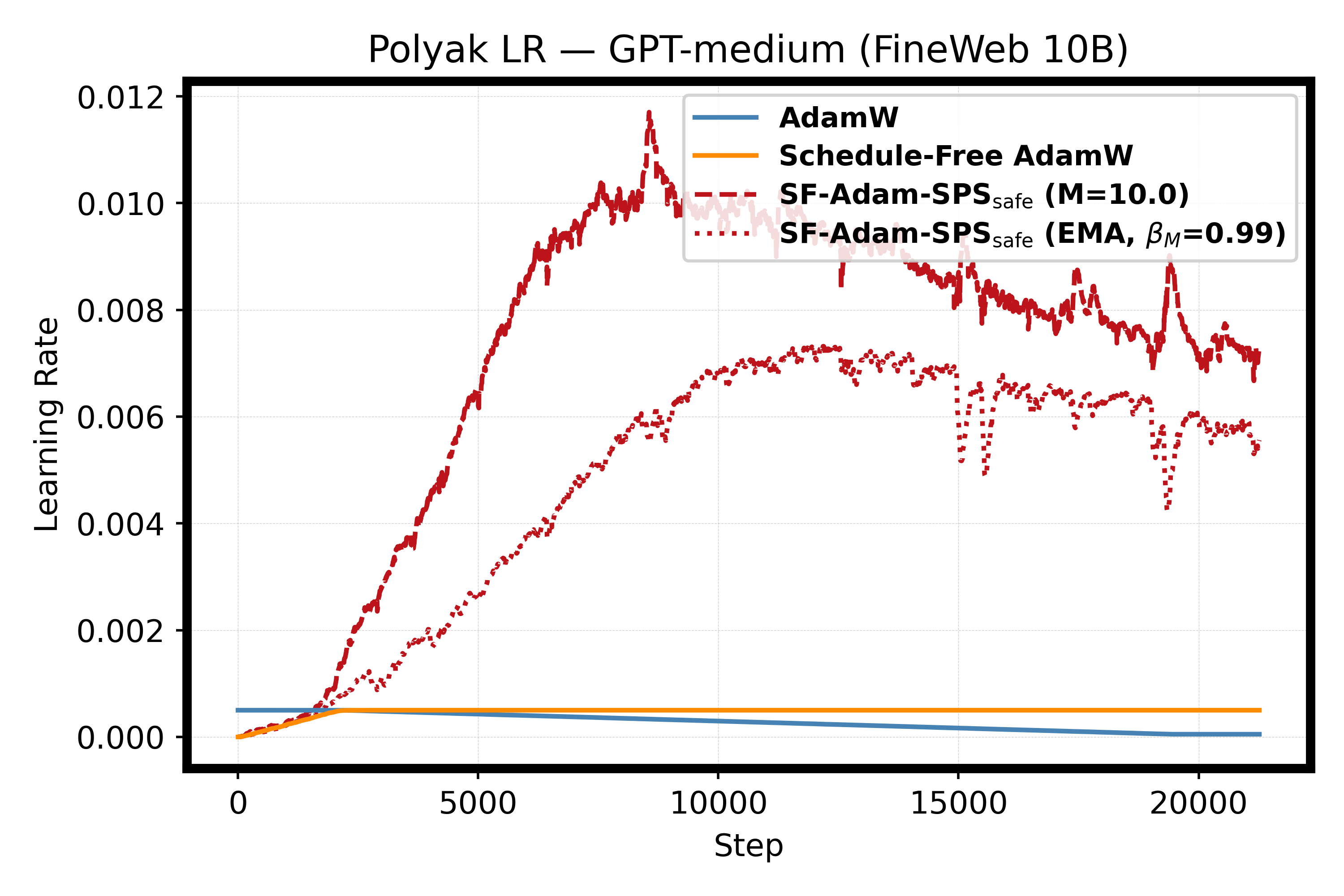}
        \caption{GPT2-medium}
        \label{fig:polyak-steps-gptmedium}
    \end{subfigure}
    \caption{Optimal learning rates computed for \ref{eq:sf-adam-sps-safe}, along with 
    learning rates for AdamW and Schedule-Free AdamW (picked for best final validation loss)}
    \label{fig:polyak-steps}
\end{figure}

\paragraph{Activation of Safeguard $M$}

\Cref{fig:polyak-M-gptsmall} shows $M$ plotted against the preconditioned gradient norms throughout training. When $M$ dominates the denominator the safeguard is active, preventing  excessively large step-sizes. With EMA-based tracking, $M$ adapts to the scale of the gradient norms, activating more selectively than a fixed $M$. As demonstrated by \Cref{fig:lr-stability-gptsmall}, in the underparameterized regime the safeguard is not necessary to improve performance and the two methods for selecting $M$ are comparable in terms of training dynamics and generalization performance.

\begin{figure}[H]
    \centering
    \includegraphics[width=\linewidth]{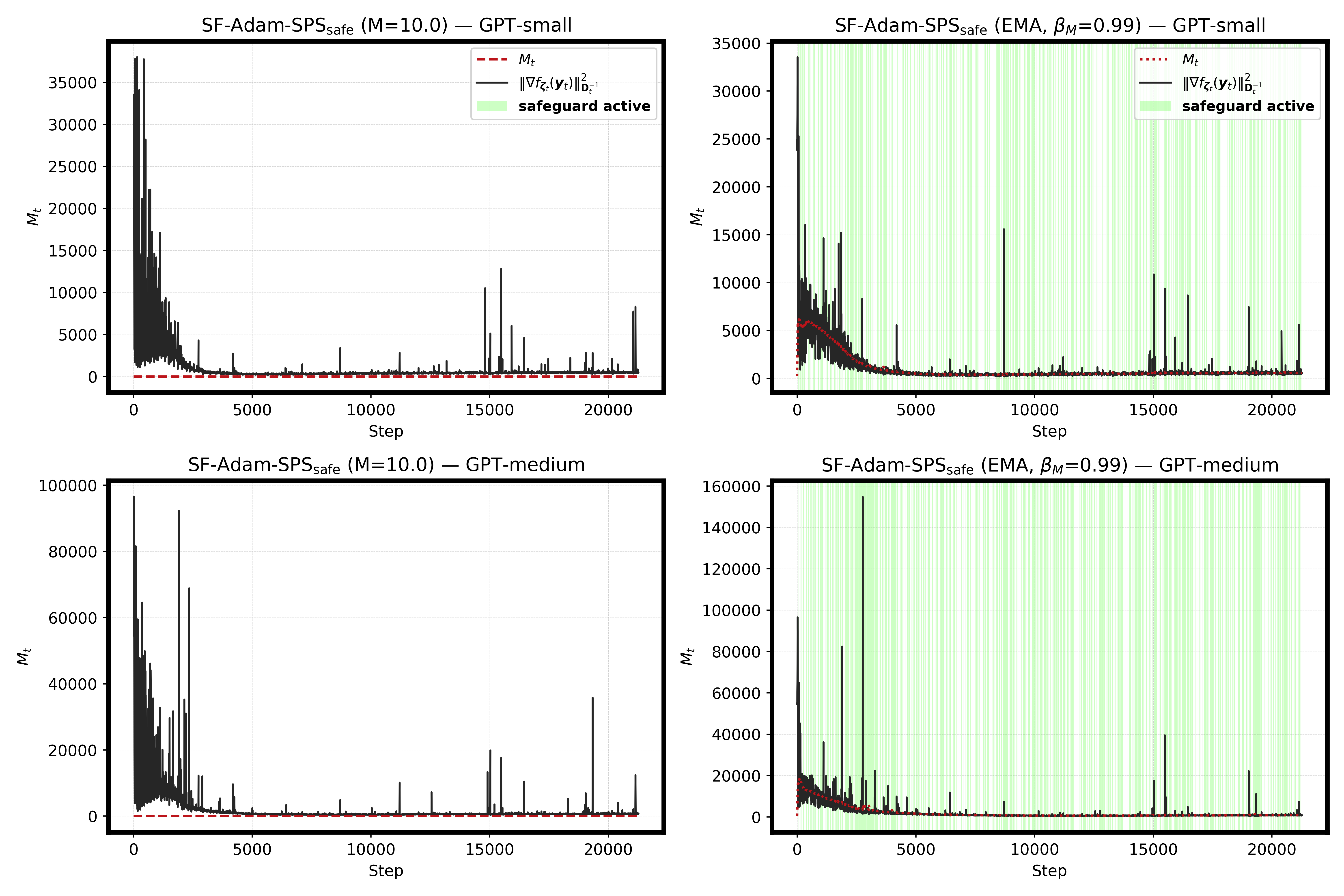}
    \caption{Safeguard parameter $M$ for \ref{eq:sf-adam-sps-safe} plotted against the 
    preconditioned gradient norms. A vertical green line indicates that 
    $M \geq \norm{\nabla f_{\bm{\zeta}_t}(\bm{y}_t)}_{\mD_t^{-1}}^2$}
    \label{fig:polyak-M-gptsmall}
\end{figure}

\end{document}

%% file: NeurIPS_26/preamble.tex
\usepackage[utf8]{inputenc}
\usepackage[T1]{fontenc}
\usepackage{microtype}

\usepackage{amsfonts,amssymb,amsmath,amsthm,amsbsy}
\usepackage{mathtools}
\usepackage{bm}
\usepackage{nicefrac}
\usepackage{mathdots}

\usepackage[dvipsnames]{xcolor}
\usepackage{colortbl}
\usepackage[colorlinks,linkcolor={blue!75!black},urlcolor=blue,citecolor={blue!75!black}]{hyperref}
\usepackage{url}

\definecolor{bgcolor}{rgb}{0.8,1,1}
\definecolor{bgcolor2}{rgb}{0.8,1,0.8}
\definecolor{niceblue}{rgb}{0.0,0.19,0.56}

\usepackage{graphicx}
\graphicspath{{figs/}}
\usepackage{caption}
\usepackage{subcaption}
\usepackage{wrapfig}
\usepackage{float}
\usepackage{booktabs}
\usepackage{multirow}
\usepackage{makecell}
\usepackage{array}
\usepackage{diagbox}
\usepackage{arydshln}
\usepackage{threeparttable}
\usepackage{rotating}

\usepackage{enumerate}
\usepackage{enumitem}

\usepackage{algorithm}
\usepackage{algpseudocode}                     

\usepackage{tikz}
\usetikzlibrary{arrows}

\usepackage{csquotes}
\usepackage{pifont}
\usepackage{lipsum}
\usepackage{setspace}
\usepackage{pdflscape}
\usepackage[capitalize,noabbrev]{cleveref}

\usepackage{mdframed}
\usepackage{thmtools}
\allowdisplaybreaks[4]

\definecolor{shadecolor}{gray}{0.90}
\declaretheoremstyle[
  headfont=\normalfont\bfseries,
  notefont=\mdseries, notebraces={(}{)},
  bodyfont=\normalfont,
  postheadspace=0.5em,
  spaceabove=6pt,
  mdframed={
    skipabove=8pt,
    skipbelow=8pt,
    hidealllines=true,
    backgroundcolor={shadecolor},
    innerleftmargin=4pt,
    innerrightmargin=4pt}
]{shaded}

\declaretheorem[style=shaded,within=section]{definition}
\declaretheorem[style=shaded,sibling=definition]{theorem}

\declaretheorem[style=shaded,sibling=definition]{lemma}

\declaretheorem[style=shaded,sibling=definition]{remark}

\usepackage[colorinlistoftodos,bordercolor=orange,backgroundcolor=orange!20,linecolor=orange,textsize=scriptsize]{todonotes}

\newcommand{\R}{\mathbb{R}}

\renewcommand{\O}{\mathcal{O}}

\newcommand{\E}[1]{\mathbb{E}\left[#1\right]}
\newcommand{\EE}[2]{\mathbb{E}_{#1}\left[#2\right]}

\newcommand{\mat}[1]{\bm{#1}}

\newcommand{\mD}{\mat{D}}

\newcommand{\mI}{\mat{I}}

\newcommand{\dotprod}[1]{\left\langle #1\right\rangle}
\newcommand{\norm}[1]{\left\| #1\right\|}

\newcommand{\cmark}{\ding{51}}
\newcommand{\xmark}{\ding{55}}

\crefname{alg}{algorithm}{algorithms}
\Crefname{alg}{Algorithm}{Algorithms}